\documentclass{article}

\usepackage{microtype}
\usepackage{xcolor}
\usepackage{graphicx}
\usepackage{subfigure}
\usepackage{booktabs} %
\usepackage{hyperref}
\usepackage{url}
\usepackage{graphicx}
\usepackage{booktabs}
\usepackage{multirow}
\usepackage{wrapfig}
\usepackage{enumitem}
\usepackage{colortbl}
\usepackage{listings}

\usepackage{hyperref}

\usepackage[accepted]{icml2024}

\usepackage{amsmath}
\usepackage{amssymb}
\usepackage{mathtools}
\usepackage{amsthm}
\newcolumntype{C}[1]{>{\centering\arraybackslash}p{#1}}
\newcolumntype{L}[1]{>{\arraybackslash}p{#1}}
\definecolor{mycolor1}{HTML}{1f77b4}
\definecolor{mycolor2}{HTML}{ff7f0e}
\definecolor{mycolor3}{HTML}{d62728}
\definecolor{dcyan}{HTML}{D62728}
\usepackage[capitalize,noabbrev]{cleveref}

\theoremstyle{plain}
\newtheorem{theorem}{Theorem}[section]

\newtheorem{lemma}[theorem]{Lemma}

\theoremstyle{definition}

\theoremstyle{remark}

\usepackage[textsize=tiny]{todonotes}

\icmltitlerunning{Domain-Knowledge-Free Diffusion-based Data Augmentation Can Enhance Unsupervised Contrastive Learning}

\begin{document}

\twocolumn[
\icmltitle{DiffAug: Enhance Unsupervised Contrastive Learning \\
with Domain-Knowledge-Free Diffusion-based Data Augmentation}

\begin{icmlauthorlist}
\icmlauthor{Zelin Zang}{westlake,damo,nus} 
\icmlauthor{Hao Luo}{damo,hupan}
\icmlauthor{Kai Wang}{nus}
\icmlauthor{Panpan Zhang}{nus}
\icmlauthor{Fan Wang}{damo,hupan}
\icmlauthor{Stan.Z Li}{westlake}
\icmlauthor{Yang You}{nus}
\end{icmlauthorlist}

\icmlaffiliation{westlake}{AI Lab, Research Center for Industries of the Future, Westlake University, China}
\icmlaffiliation{damo}{DAMO Academy, Alibaba Group}
\icmlaffiliation{hupan}{Hupan Lab, Zhejiang Province}
\icmlaffiliation{nus}{National University of Singapore}

\icmlcorrespondingauthor{Stan.Z Li}{Stan.Z.Li@westlake.edu.cn}

\icmlkeywords{Machine Learning, ICML}

\vskip 0.3in
]

\printAffiliationsAndNotice{}  %

\begin{abstract}
    Unsupervised Contrastive learning has gained prominence in fields such as vision, and biology, leveraging predefined positive/negative samples for representation learning. Data augmentation, categorized into hand-designed and model-based methods, has been identified as a crucial component for enhancing contrastive learning. However, hand-designed methods require human expertise in domain-specific data while sometimes distorting the meaning of the data. In contrast, generative model-based approaches usually require supervised or large-scale external data, which has become a bottleneck constraining model training in many domains. To address the problems presented above, this paper proposes DiffAug, a novel unsupervised contrastive learning technique with diffusion mode-based positive data generation. DiffAug consists of a semantic encoder and a conditional diffusion model; the conditional diffusion model generates new positive samples conditioned on the semantic encoding to serve the training of unsupervised contrast learning. With the help of iterative training of the semantic encoder and diffusion model, DiffAug improves the representation ability in an uninterrupted and unsupervised manner. Experimental evaluations show that DiffAug outperforms hand-designed and SOTA model-based augmentation methods on DNA sequence, visual, and bio-feature datasets. The code for review is released at \url{https://github.com/zangzelin/code_diffaug}.
\end{abstract}

\section{Introduction}

Contrastive learning, as shown by many studies~\citep{he_momentum_2020,chen_simple_2020,cui2021parametric,wang2022contrastive,assran2022masked,zang2023boosting}, has become important in areas like vision~\citep{he2021masked, zang2022dlme}, natural language processing~\citep{rethmeier2023primer}, and biology~\citep{yu2023enzyme,krishnan2022self}. The key to contrastive learning~(CL) lies in designing appropriate data augmentation methods. Many studies~\citep{tian2020makes,zhang2022rethinking, peng2022crafting, zhang2023adaptive} have found that data augmentation helps CL by making it more robust and preventing model collapse problems.

\begin{figure*}
    \centering
    \includegraphics[width=0.99\textwidth]{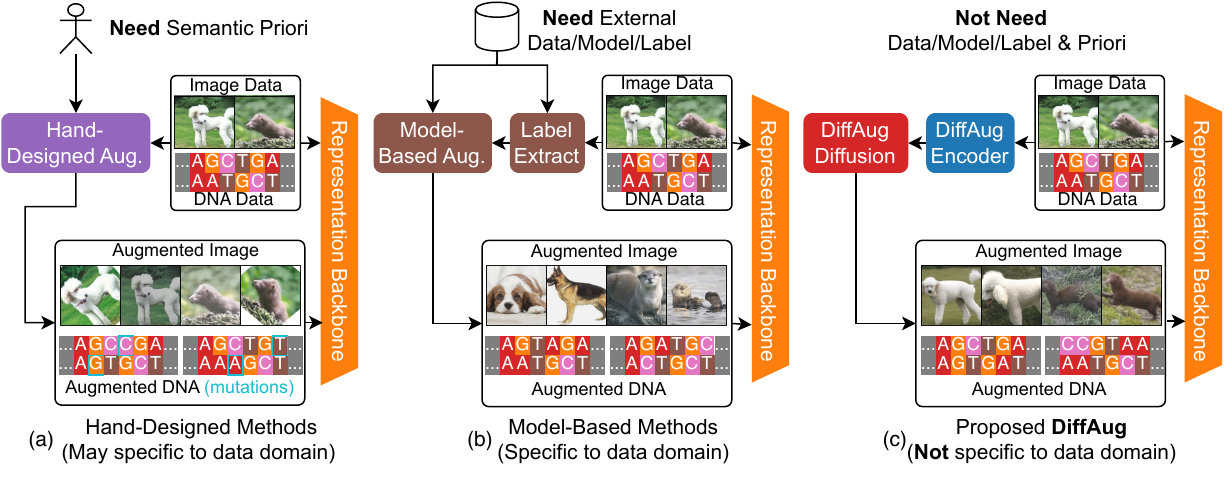}
    \caption{\textbf{Comparison of DiffAug with existing augmentation strategy.} (a) Hand-designed augmentation is based on human priori that generate new data with different feature but semantically similar semantic. (b) Model-based augmentation methods generate new data with the same labels by training generative models with large amount of data, labels. These methods often require large amounts of data and target specific data domains. (c) DiffAug attempts to reduce the dependence on external data and prior knowledge through iterative training with encoders and diffusion. Expanding the application areas of unsupervised CL.}
    \label{fig_overview}
\end{figure*}

Data augmentation falls into two main types: \textit{hand-designed methods} and \textit{model-based methods}~\citep{xu2023comprehensive}.
Hand-designed methods require humans to understand the meaning of the data and then change the input features while maintaining or extending that meaning. For example, several methods in visual tasks~(such as color change~\citep{yan2022rgb}, random cropping~\citep{cubuk2020randaugment}, and rotation~\citep{maharana2022review}) and DNA sequence representation tasks~(such as mutations, insertion, and noise~\cite{lee2023evoaug}) are used to aid in model training. 
However, the problem is that the above techniques must be more data-specific. For some data (genes or proteins or others), it isn't easy to understand due to the complexity of its meaning. Consequently, it isn't easy to design a good augmentation strategy. Semantics-independent augmentation methods such as adding noise~\citep{huang2022novel} and random hiding~\citep{theodoris2023transfer} are used, but only sometimes with significant results.
Another issue with hand-designed methods is their inability to subtly alter the semantics of the data. For instance, a minor mutation in a DNA sequence can lead to significant semantic changes, akin to a gene mutation (as illustrated in Figure.~\ref{fig_overview}). As a result, more positive/negative samples are needed to distribute these risks to obtain a stable representation. It is also challenging to train CL models with fewer samples for certain domains where data acquisition is costly, such as biology.

Given the challenges mentioned earlier, model-based methods~(generative models) based on deep learning are used to create better data. In the vision domain, techniques using VAE~\citep{VAE}, GAN~\citep{GAN}, and diffusion models~\citep{DDPM,ImprovedDDPM,Imagen,GLIDE,unCLIP} have been developed to improve model training. For supervised learning, several studies have received attention. \citet{du2023dream} proposed the DREAM-OOD framework, which uses diffusion models to generate photo-realistic outliers from in-distribution data for improved OOD detection. \citet{zhang2023expanding} developed the Guided Imagination Framework (GIF) using generative models like DALL-E2 and Stable Diffusion for dataset expansion, enhancing accuracy in both natural and medical image datasets. The detailed related works are in Appendix.\ref{app_releadwork}. 

However, there are concerns about these methods, especially about their diversity and how well they generalize. Moreover, most of these generative models are trained with supervision or need much external data. This makes them less suitable for areas like DNA sequence and bio-feature data~(in Figure.~\ref{fig_overview}). This leads to an important question: \textit{Is it possible to design a data augmentation framework to enhance unsupervised CL in different domains without requiring expert knowledge or additional data?}

We introduce DiffAug, a novel diffusion mode-based positive data generation technique for unsupervised CL to address the posed problem. DiffAug eliminates the need for training labels. Instead, we employ a semantic estimator to gauge the semantics of the input data, subsequently guiding the augmentation process. At its core, DiffAug operates through two synergistic modules: a semantic encoder and a diffusion generator. Utilizing a soft contrastive loss, the semantic encoder crafts latent representations that act as guiding vectors for the diffusion generator. This generator then methodically produces augmented data in the input space, ensuring varying levels of semantic consistency based on the guiding vectors and specific adjustable hyperparameters. 

\begin{figure*}
    \centering
    \includegraphics[width=0.99\linewidth]{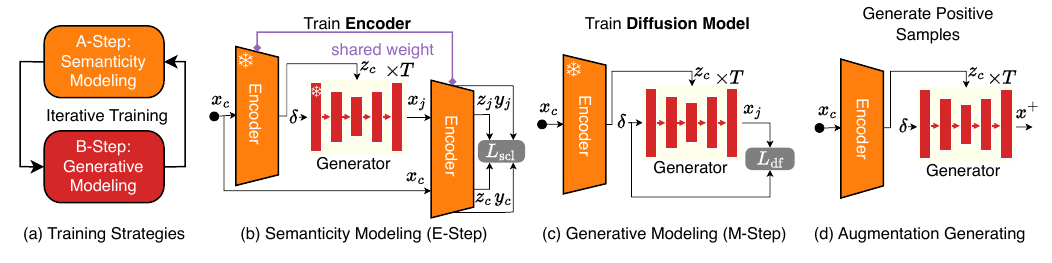}
    \caption{
        \textbf{The DiffAug framework and training strategy.} DiffAug includes a semantic encoder $\text{Enc}(\cdot|\theta)$ and a diffusion generator $\text{Gen}(\cdot|\phi)$. (a) shows how $\text{Enc}(\cdot|\theta)$ and $\text{Gen}(\cdot|\phi)$ are interative trained. (b) and (c) show how to calculate the loss functions. (d) shows how to generate new augmentation data with the trained model.
    }
    \label{fig_method}
\end{figure*}

We demonstrate DiffAug pioneering on DNA sequence and biometric datasets, and pioneering on highly competitive visual datasets. Our findings indicate that the proposed method can produce sensible data augmentations, subsequently enhancing the performance of unsupervised CL that utilizes these augmentations. Notably, {DiffAug} performs superior classification and clustering tasks compared to all benchmark methods. The primary contributions of this paper are:
(a) We introduce DiffAug, a novel unsupervised CL technique with diffusion mode-based positive data generation. DiffAug's data augmentation improves traditional domain-specific hand-designed data augmentation strategy.
(b) {DiffAug} operates independently of external data or manually designed rules. Its versatility allows for integration with various models, encompassing domains like vision or biology studies.
(b) The experimental results show the efficacy of {DiffAug} in enhancing the performance of CL in different tasks. This suggests that DiffAug can generate positive sample data unsupervised, which in turn promotes the development of unsupervised learning techniques.

\section{Methods}
In the context of unsupervised data augmentation, the training dataset providing potential semantic categories is denoted as $\mathcal{D}_{t}=\{\mathbf{x}_i\}_{i=1}^N$, where $N$ is the size of the training set. To boost the training efficiency of unsupervised contrastive learning with positive samples generated by the diffusion model, a novel framework called DiffAug is proposed.

\subsection{Preliminaries of Contrastive Learning}
\textbf{Contrastive Learning. }
Contrastive learning learns visual representation via enforcing the similarity of the positive pairs and enlarging distance of negative pairs. Formally, loss is defined as, $\mathcal{L}_{\mathrm{cl}}=$
\begin{equation}
    -\!\log\!{\mathcal{Q} \left(\mathbf{z}_{i} ,\mathbf{z}_{i}^{+}\right)} \!+\! \log
    [
        {\mathcal{Q} \left(\mathbf{z}_{i}, \mathbf{z}_{i}^{+}\right)\!\!+\!\!\sum_{\!\!\mathbf{z}_{i}^{-} \in V^{-}} \mathcal{Q} \left(\mathbf{z}_{i}, \mathbf{z}_{i}^{-}\right)}
    ]
\end{equation}
where $\mathbf{z}_i$ is the low dimensional embedding $\mathbf{z}_i = \text{Enc}^\mathrm{cl}(\mathbf{x}_i)$, $Q \left(\mathbf{z}_{i} ,\mathbf{z}_{i}^{+}\right)$ indicates the similarity between positive pairs while $Q \left(\mathbf{z}_{i} ,\mathbf{z}_{i}^{-}\right)$ is the similarity between negative pairs.
For the traditional scheme, in the computer vision domain, data augmentation methods such as random cropping~\citep{cubuk2020randaugment} or data Mixup~\citep{Hongyi_mixup_2017} are used to generate new positive data. The negative samples $v_{i}^{-}$ are sampled from negative distribution $V^{-}$.

\textbf{Soft Contrastive Learning.} To address the performance degradation due to view noise in contrastive learning and to accomplish unsupervised learning on smaller scale datasets,~\citet{zang2023boosting} designed soft contrastive learning, which smoothes sharp positive and negative sample pair labels by evaluating the credibility of the sample pairs. Consider the loss form for multiple positive samples and multiple negative samples as,
\begin{equation}
    \begin{aligned}
        \mathcal{L}_\mathrm{scl}(\mathbf{y}_c, \!\mathbf{y}_j, \!\mathbf{z}_c, \!\mathbf{z}_j\!)\! &= \!-\! \!\sum_{j=1}^{\mathcal{B}} \{
              \! {\mathcal{P}\!(\mathbf{y}_c, \mathbf{y}_j)\!} \log\!\left(\!\mathcal{Q}(\mathbf{z}_c, \mathbf{z}_j)\right)\!+\!\\
              &{(1\!-\!\mathcal{P}(\mathbf{y}_c, \mathbf{y}_j))} \log\left(1\!-\!\mathcal{Q}(\mathbf{z}_c, \mathbf{z}_j)\right) 
        \}, \\
        \mathcal{P}(\mathbf{a}, \mathbf{b}) &=
        \left(1 + \mathcal{H}_{ij}\left(e^\beta-1\right)\right) \mathcal{Q}(\mathbf{a}, \mathbf{b}),
    \end{aligned}
    \label{eq_SCL}
\end{equation}
where the $\mathbf{y}_i, \mathbf{z}_i$ are the high dimensional embedding and low dimensional embedding $\mathbf{y_i}, \mathbf{z_i} = \text{Enc}(\mathbf{x_i})$. The $\mathcal{P}(\mathbf{a}, \mathbf{b})$ is soft learning weight and calculated by the positive/negative pair indicator $\mathcal{H}_{cj}$. The hyper-parameter $\beta \in [0,1]$ introduces prior knowledge of data augmentation relationship $\mathcal{H}_{cj}$ into the model training. Details of contrastive and soft contrastive learning are in Appendix.~\ref{app_scl}.

\begin{table*}[ht]
	\centering
	\caption{
		\textbf{Comparison of Linear probing results on DNA sequence datasets.} The compared methods  including SOTA DNA sequence methods~(DNA-BERT~\citep{ji2021dnabert}, NT~\citep{dalla2023nucleotide}, Hyena~\citep{nguyen2023hyenadna}) and contrastive methods with human-designed DNA-augmentation. The `AVE' represents the average performance. The best results are marked in \textbf{bold}.
	}
	\footnotesize
	\begin{tabular}{L{2.4cm}||C{1.1cm}C{1.1cm}C{1.1cm}C{1.3cm}C{1.1cm}C{1.1cm}C{1.3cm}C{1.2cm}|c}
		\toprule
		\multirow{2}{*}{Datasets}                  & \multicolumn{8}{c}{Genomic Benchmarks~\cite{grevsova2023genomic}}                                                                                                                                                                                                                                                                                                                                                   \\
		          & MoEnEn                                   & CoIn                                     & HuWo                                     & HuEnCo                            & HuEnEn                                   & HuEnRe                                   & HuNoPr                                   & HuOcEn                                   & AVE                                      \\\midrule
		CNN               & 69.0                                     & 87.6                                     & 93.0                                     & 58.6                              & 69.5                                     & 93.3                                     & 84.6                                     & 68.0                                     & 76.7                                     \\
		DNA-BERT          & 69.4                                     & 92.3                                     & 96.3                                     & \textbf{74.3}                     & 81.1                                     & 87.7                                     & 85.8                                     & 73.1                                     & 82.5                                     \\
		NT                & 70.2                                     & 90.0                                     & 92.3                                     & 71.5                              & 80.8                                     & 87.9                                     & 84.0                                     & 77.2                                     & 81.7                                     \\
		Hyena             & 80.9                                     & 89.0                                     & 96.4                                     & 73.2                              & 88.1                                     & 88.1                                     & \textbf{94.6}                                     & 79.2                                     & 86.2                                     \\
		\midrule
		SSL+Translocation & 83.8                                     & 88.2                                     & 95.5                                     & 73.8                              & 77.4                                     & 88.2                                     & 84.7                                     & 52.5                                     & 80.5                                     \\
		SSL+RC            & 84.5                                     & 88.3                                     & 95.8                                     & 71.9                              & 82.3                                     & 86.8                                     & 91.0                                     & 74.4                                     & 84.3                                     \\
		SSL+Insertion     & 80.9                                     & 89.8                                     & 96.6                                     & 73.7                              & 83.4                                     & 87.1                                     & 91.8                                     & 77.3                                     & 85.0                                     \\
		SSL+Mixup         & 80.9                                     & 89.4                                     & 96.2                                     & 73.2                              & 85.2                                     & 88.6                                     & 91.6                                     & 77.9                                     & 85.4                                     \\ \midrule
		DiffAug           & \textbf{86.0}\textcolor{dcyan}{{(+1.5)}} & \textbf{94.9}\textcolor{dcyan}{{(+2.6)}} & \textbf{96.8}\textcolor{dcyan}{{(+0.2)}} & {74.0}\textcolor{dcyan}{{(-0.3)}} & \textbf{94.9}\textcolor{dcyan}{{(+6.8)}} & \textbf{91.8}\textcolor{dcyan}{{(+3.2)}} & {94.5}\textcolor{dcyan}{{(-0.1)}} & \textbf{79.9}\textcolor{dcyan}{{(+0.7)}} & \textbf{89.1}\textcolor{dcyan}{{(+2.9)}} \\
		\bottomrule
	\end{tabular}
	\label{tab_exp_DNA}
\end{table*}

\subsection{DiffAug Design Details and Training Strategies}
\textbf{DiffAug Framework.} DiffAug accomplishes the tasks of \textit{positive sample generation} and \textit{data representation} by iterating the two modules over each other~(in Figure.~\ref{fig_method}). DiffAug consists of two main modules, a semantic encoder $\text{Enc}(\cdot|\theta)$ and a diffusion generator $\text{Gen}(\cdot|\phi)$, where $\theta$ and $\phi$ are model parameters.
The $\text{Enc}(\cdot|\theta)$ maps the input data $\mathbf{x}_{i}$ to the discriminative latent space $\mathbf{v}_{i}$, and the generator $\text{Gen}(\cdot|\phi)$ generates new data with a semantic vector $\mathbf{v}_{i}$.
Similar to the Expectation maximization algorithm~\citep{gupta2011theory}, the semantic encoder $\text{Enc}(\cdot|\theta)$ and the diffusion generator $\text{Gen}(\cdot|\phi)$ are trained in turn by two different loss functions (see Figure.~\ref{fig_method}(a) and Figure.~\ref{fig_method}(b)).

\begin{algorithm}[t]
	\caption{The DiffAug Training Algorithm: }
	\footnotesize
	\label{alg:algorithm}
	\textbf{Input}:
	Data: $\mathcal{D}_{t}=\{\textbf{x}_c\}_{i=1}^N$,
	Learning rate: $\eta$,
  E or M State: S,
	Batch size: $B$,
	Network parameters: $\theta, {\phi}$, \\
	\textbf{Output}:
  Updateed Parameters: $\theta, {\phi}$.

	\begin{algorithmic}[1] %
		\WHILE{{$b=0$; $b<[ |\mathcal X| /B]$; $b$++}}
		\STATE $\mathbf{x}_c \sim \mathcal{D}_{t}$;
		\textcolor{blue}{\qquad\qquad\  \# Sample the centering data} \\
		\STATE $\mathbf{y}_c, \mathbf{z}_c \!\leftarrow\! \text{Enc}(\mathbf{x}_c|\theta)$;
		\textcolor{blue}{ \# Generate frozen condition vector} \\
    \IF {S==\text{B-Step}} 
		 \STATE ${L}_{1}\!\leftarrow\! L_{\text{df}}(\mathbf{x}_c, \text{SG}(\mathbf{z}_c)|\phi) $ by Eq.~(\ref{eq_df});  
		 \STATE $\phi \!\leftarrow\!  \phi - \eta \frac{ \partial \mathcal{L}_{\text{1}} }{\partial \phi}$,  \textcolor{blue}{ \ \ \ \ \# Calculate diffusion loss} \\
    \ELSE
     \STATE $ \mathcal{B}_c = \{\mathbf{x}_1, \cdots, \mathbf{x}_\mathcal{B}| \mathbf{x}_j \! \sim  \!\mathcal{D}_{t} \mathbf{\ if \ } \mathcal{H}_{ij} = 0 ; \mathbf{x}_j  \!\sim \! \text{Aug}(\mathbf{x}_c) \mathbf{\ else\ } \} $;
	 \ \ \textcolor{blue}{\# Generate/sample data} \\
		\STATE $\mathcal{Y}=\{\mathbf{y}_1, \cdots, \mathbf{z}_j, \cdots, \mathbf{y}_\mathcal{B}\}$, \STATE $\mathcal{Z}=\{\mathbf{z}_1, \cdots, \mathbf{z}_j, \cdots, \mathbf{z}_\mathcal{B}\},  \mathbf{y}_j, \mathbf{z}_j = \text{Enc}(\mathbf{x}_j|\theta)$
		\STATE ${L}_{2} \!\leftarrow\! L_{\text{scl}}(\mathcal{Y}, \mathcal{Z})$ by Eq.~(\ref{eq_SCL}); \STATE $\theta \!\leftarrow\! \theta - \eta \frac{ \partial \mathcal{L}_{\text{2}} }{\partial \theta}$
		\textcolor{blue}{ \ \ \ \ \ \# Calculate scl loss} \\
    \ENDIF
		\ENDWHILE
	\end{algorithmic}
	\label{alg_1}
\end{algorithm}

\textbf{Semanticity Modeling~(A-Step).} In the semanticity modeling step, given a central data $\mathbf{x}_c$, we generate a background set $\mathcal{B}_{c}=\{\mathbf{x}_1, \cdots, \mathbf{x}_j,\cdots,\mathbf{x}_{\mathcal{N}_b}\}$,
\begin{equation}
    \begin{aligned}
        \left\{
        \begin{aligned}
            \mathbf{x}_j & \sim \mathcal{D}_{t}           & \text{if} \ \ \  \mathcal{H}_{cj}=0 \\
            \mathbf{x}_j & \sim A_\text{ug}(\mathbf{x}_c) & \text{if} \ \ \ \mathcal{H}_{cj}=1
        \end{aligned}
        \right.
        \label{eq_CL}
    \end{aligned}
\end{equation}
where ${N}_b$ is the number of background data points.
The $\mathcal{H}_{cj}=0$ indicates $\mathbf{x}_j$ is sampled from the dataset $\mathcal{D}_{t}$, and $\mathbf{x}_c$ and $\mathbf{x}_j$ are negative pair.
Meanwhile, $\mathcal{H}_{cj}=1$ indicates $\mathbf{x}_c$ and $\mathbf{x}_j$ are positive pair and $\mathbf{x}_j$ is sampled from data augmentation. For details, new positive data are generated by the diffusion model according to DDPM~\citep{DDPM},
\begin{equation}
    \begin{aligned}
        \mathbf{x}_j & = \text{Gen}(\delta, \mathbf{z}_c|\phi^*),
        \mathbf{y}_c, \mathbf{z}_c=\text{Enc}(\mathbf{x_c}|\theta^*), \\
    \end{aligned}
    \label{eq_gen_c}
\end{equation}
where $\text{Gen}(\delta, \mathbf{z}_c|\phi^*)$ is the generation process of the diffusion model, and the generating details are in Eq.~(\ref{eq_gen}).
The $\delta\sim \mathcal{N}(0, \mathbf{1})$ is the random initialized data, and $\mathbf{z}_c$ is a conditional vector.
The $*$ in $\phi^*$ and $\theta^*$ means the parameter is frozen.
To avoid unstable positive samples from untrained generative models, training starts exclusively with traditional data augmentation tools, and then, the data generated by DiffAug is replaced with data generated by DiffAug, with a replacement probability of the hyperparameter $\lambda$, an oversized $\lambda$ introduce toxicity, which we will discuss in Sec.~\ref{lambda}. 
We update the parameter of the semantic encoder with the soft contrastive learning loss,
\begin{equation}
    \begin{aligned}
    &\theta = \theta - \eta \sum_{\mathbf{x}_j\in \mathcal{B}_c} \left\{\mathcal{L}_\mathrm{scl}(\mathbf{y}_c, \mathbf{z}_c, \mathbf{y}_j, \mathbf{z}_j)\right\}, \\ 
    &\text{ where } \mathbf{y}_j, \mathbf{z}_j = \text{Enc}(\mathbf{x}_j|\theta),
    \end{aligned}
\end{equation}
where the $\eta$ is the learning rate, and the $\mathcal{L}_\mathrm{scl}(\mathbf{y}_c, \mathbf{z}_c, \mathbf{y}_j, \mathbf{z}_j)$ is in Eq.~(\ref{eq_SCL}).

\textbf{Generative Modeling~(B-Step).} In the generative modeling step, the conditional diffusion generator $\text{Gen}(\cdot| \phi)$ is trained by the vanilla diffusion loss $\mathcal{L}_\text{df}(\mathbf{x}_c, \mathbf{z}_c| \phi)$ \citep{DDPM},
\begin{equation}
    \phi\! =\! \phi \!- \!\eta \sum _{t=1} ^{T}
    \!\left\{\!
    \left\|
    \delta-g_\phi
    \left(
    \sqrt{\bar{\alpha}_t} \widetilde{\mathbf{x}}_c^{t} \!+\! \sqrt{1-\bar{\alpha}_t}, t, \mathbf{z}_c
    \right)
    \right\|_{2}^{2}
    \!\right\}\!, \\
    \label{eq_df}
\end{equation}
where the conditional vector $\mathbf{z}_c$ is generated from the semantic encoder in Eq.(\ref{eq_gen_c}). The $g_\phi(\cdot)$ is the conditional diffusion neural network. The $\alpha_t$ is the noise parameter in the diffusion process, and $\bar{\alpha}_t = 1-\alpha_t$. The $\widetilde{\mathbf{x}}_c^{t}$ is the intermediate data in the diffusion process, and the $\widetilde{\mathbf{x}}_c^{0} = \mathbf{x}_c$. $T$ is the time step of the generation process. When $g_\phi(\cdot)$ is trained, the detailed generating process is, 
\begin{equation}
    \begin{aligned}
        \text{Gen}(\delta, \mathbf{z}_c|\phi^*)\!&=\!\left\{\!
        \widetilde{\mathbf{x}}^{0} \ | \ \widetilde{\mathbf{x}}^{t-1} \!=\!
        \frac{1}{\sqrt{\alpha_t}}
        \left(
            \widetilde{\mathbf{x}}^{t} \!-\!
            \hat{\alpha}
            \right) \!+\! \sigma_t \mathcal{N}(0,1),
            \!\right\}\!,\\
        \hat{\alpha} &= \frac{1-\alpha_t}{\sqrt{1\!-\!\bar{\alpha}_t}} g_\phi (\widetilde{\mathbf{x}}^{t}, t, \mathbf{z}^{*}_c),
    \end{aligned}
    \label{eq_gen}
\end{equation}
where $t\in\{T,\!\cdots\!,\! 1\}$, the $g_\phi(\cdot)$ is a neural network approximator intended to predict $\delta$ with $\widetilde{\mathbf{x}}$ and the condition vector $\mathbf{z}_c^*$.

\textbf{Augmentation Generation.} Given the trained semantic encoder $\text{Enc}(\cdot|\theta)$ and diffusion
generator $D(\cdot)$, and DiffAug generate new augmented data $\mathbf{x}_{i}^{+}$ from any input data $\mathbf{x}_{i}$.
\begin{equation}
    \mathbf{x}_{i}^{+} = \text{Gen}(\delta|\mathbf{z}_{i}), \ \ \mathbf{y}_{i}, \mathbf{z}_{i} = \text{Enc}(\mathbf{x}_{i}).
\end{equation}
Meanwhile, DiffAug's semantic encoder can be seen as a feature extractor. It is considered to have good discriminative performance because it is trained simultaneously as the diffusion
generator. 

\section{Results}
\begin{table*}[t]
	\centering
	\small
	\caption{\textbf{Comparison of the Linear probing results on Bio-feature dataset.}
		dimensional reduction~(DR) methods and contrastive learning methods are list in the table. The DR methods are widely used on Bio-feature analysis including TopoAE~\citep{moor_topological_2019}, PaCMAP~\citep{Yingfan_pacmap}, and hNNE~\citep{sarfraz2022hierarchical}.
		CL methods wtih mixup augmentation including Simclr~\citep{chen_simple_2020}, BYOL~\citep{grill_bootstrap_2020}, MoCo~\citep{he_momentum_2020}, and DLME~\citep{zang2022dlme}.
	}
	\begin{tabular}{@{}cc||cccccc|cccc}
		\toprule
		             & Method Type & {GA1457}                                   & {SAM561}                                   & {MC1374}                                   & {HCL500}                                   & {WARPARIOP}                                & {PROT579}                              & AVE                                        \\                                          \midrule
		TopoAE       & DR          & 74.6                                       & 72.4                                       & 61.3                                       & 56.0                                       & 73.0                                       & 88.3                                   & 70.9                                       \\
		PaCMAP       & DR          & 85.3                                       & 83.7                                       & 61.3                                       & 36.2                                       & 76.9                                       & 87.2                                   & 71.7                                       \\
		hNNE         & DR          & 77.4                                       & 83.8                                       & 62.3                                       & 62.2                                       & 72.5                                       & 82.4                                   & 73.4                                       \\ \midrule
		Simclr+Mixup & CL          & 84.8                                       & 82.4                                       & 62.3                                       & 61.7                                       & 74.3                                       & 74.7                                   & 73.3                                       \\
		BYOL+Mixup   & CL          & 82.8                                       & 73.3                                       & 60.9                                       & 61.3                                       & 72.6                                       & 70.5                                   & 70.2                                       \\
		MoCo+Mixup   & CL          & 84.2                                       & 83.1                                       & 70.2                                       & 61.7                                       & 82.8                                       & 84.2                                   & 77.7                                       \\
		DLME+Mixup   & CL          & 85.7                                       & 83.6                                       & 71.4                                       & 62.3                                       & 83.5                                       & 84.5                                   & 78.5                                       \\ \midrule
		DiffAug      & CL          & {\textbf{92.7}\textcolor{dcyan}{{(+7.0)}}} & {\textbf{89.3}\textcolor{dcyan}{{(+5.5)}}} & {\textbf{71.8}\textcolor{dcyan}{{(+0.4)}}} & {\textbf{64.7}\textcolor{dcyan}{{(+2.4)}}} & {\textbf{84.8}\textcolor{dcyan}{{(+1.3)}}} & {\bf{91.2}}\textcolor{dcyan}{{(+2.9)}} & {\textbf{82.4}\textcolor{dcyan}{{(+3.9)}}} \\
		\bottomrule
	\end{tabular}
	\label{tab_exp_scrna}
\end{table*}

\begin{table}[t]
	\centering
	\footnotesize
	\caption{\textbf{Comparison of Linear probing results on vision dataset.}
		The SimC.+Mix. and Mo.V2+Mix. are SimCLR and MoCoV2 with Mixup data augmentation which processed by \citet{zhang2022m}. The SimC./Mo.V2+VAE/GAN means SimCLR/MoCoV2 with VAE/GAN generative model as data augmentation.}
	\begin{tabular}{@{}C{1.8cm}||C{1.1cm}C{1.1cm}C{1.1cm}C{1.1cm}}
		\toprule
		Datasets    & {CF10}                                     & {CF100}                                    & {STL10}                                    & {TINet}                                    \\
		\midrule
		SimCLR      & {89.6}                                     & {60.3}                                     & {89.0}                                     & {45.2}                                     \\
		Mo.V2       & {86.7}                                     & {56.1}                                     & {89.1}                                     & {47.1}                                     \\
		BYOL        & {92.0}                                     & {62.7}                                     & {91.8}                                     & {46.1}                                     \\
		SimSiam     & {91.6}                                     & {64.7}                                     & {89.4}                                     & {43.0}                                     \\
		DINO        & {91.8}                                     & {67.4}                                     & {91.7}                                     & {44.2}                                     \\
		\midrule
		SimC.+Mixup & {90.9}                                     & {62.9}                                     & {89.6}                                     & {---}                                      \\
		Mo.V2+Mixup & {91.5}                                     & {62.7}                                     & {90.1}                                     & {---}                                      \\ \midrule
		SimC.+VAE   & {89.6}                                     & {64.2}                                     & {91.7}                                     & {46.0}                                     \\
		Mo.V2+VAE   & {89.3}                                     & {65.9}                                     & {91.2}                                     & {43.3}                                     \\
		SimC.+GAN   & {90.0}                                     & {64.3}                                     & {89.9}                                     & {44.6}                                     \\
		Mo.V2+GAN   & {91.1}                                     & {62.9}                                     & {91.2}                                     & {43.6}                                     \\ \midrule
		DiffAug     & {\textbf{93.4}\textcolor{dcyan}{{(+1.6)}}} & {\textbf{69.9}\textcolor{dcyan}{{(+2.5)}}} & {\textbf{92.5}\textcolor{dcyan}{{(+0.8)}}} & {\textbf{49.7}\textcolor{dcyan}{{(+2.1)}}} \\ \bottomrule
	\end{tabular}
	\label{tab_exp_image}
\end{table}

We conduct experiments on various datasets, including DNA sequences, vision, and bio-feature datasets. We aim to demonstrate that Diffaug can operate effectively and facilitate improvements across diverse domains.

\subsection{Comparations on DNA Sequence Datasets}
\label{sub_dnadeq}
First, we demonstrate the efficacy of DiffAug in improving DNA sequence representation and classification. DNA sequence representation is challenging for contrastive learning because one cannot easily design data augmentation manually by visualizing and understanding the data. \citet{lee2023evoaug} explores how various natural genetic alterations can enhance model training performance.

\textbf{Test Protocols.} 
Our experiments utilize the Genomic Benchmarks~\cite{grevsova2023genomic}, encompassing datasets that target regulatory elements (such as promoters, enhancers, and open chromatin regions) from three model organisms: humans, mice, and roundworms\footnote{https://github.com/ML-Bioinfo-CEITEC/genomic\_benchmarks.}. We adopted a methodology akin to that of Hyena-DNA~\cite{nguyen2023hyenadna} for evaluating linear-test performance. 
To mitigate the influence of the pre-training dataset, we exclusively employed the training data from Genomic Benchmarks for self-supervised pre-training and fine-tuning, followed by an evaluation of the test dataset. The comparison includes various methods such as CNN (as per Genomic Benchmarks), DNA-BERT~\cite{ji2021dnabert}, NT~\cite{dalla2023nucleotide}, and Hyena~\cite{nguyen2023hyenadna}. Both DiffAug and `SSL+' leverage Hyena-tiny\footnote{https://huggingface.co/LongSafari/hyenadna-tiny-1k-seqlen} backbone, replacing Hyena's pre-training approach with a unique pre-training methodology. 
`SSL+' means the model is pre-trained with SimCLR~\cite{chen_simple_2020} with the augmentation of natural DNA augmentation strategies in \cite{lee2023evoaug}. Further details on the dataset are provided in Appendix~\ref{app_dna_exp}.

\textbf{Analysis.}
DiffAug outperforms competing methods in eight evaluations across four datasets, achieving performance improvements ranging up to 6.8\%. Notably, DiffAug demonstrates several significant benefits, particularly in classification metrics: (a) DiffAug enhances the sequence model's performance on classification accuracy, surpassing traditional DNA augmentation approaches. (b) By learning distributional knowledge from the training data, DiffAug facilitates data augmentation with minimal human intervention, potentially enabling the generation of more stable enhanced samples.

\begin{table*}[t]
	\caption{\textbf{Ablation study of the semantic encoder includes DiffAug's encoder is necessary and can efficiently generate conditional vectors.} Linear-tests performance of different ablation setups on on vision dataset and biological dataset. }
	\centering
	\small
	\begin{tabular}{l||C{1.00cm}C{1.00cm}C{1.00cm}C{1.00cm}|C{1.00cm}C{1.00cm}C{1.00cm}C{1.00cm}C{1.00cm}}
		\toprule
		\multirow{2}{*}{Datasets}                                                & \multicolumn{4}{c|}{Vision Datasets} & \multicolumn{4}{c}{Bio-feature Datasets}                                                                                                                   \\
		                                                                         & \multirow{1}{*}{CF10}                & \multirow{1}{*}{CF100}                   & \multirow{1}{*}{STL10} & \multirow{1}{*}{TINet} & GA1457        & SAM561        & MC1374        & HCL500        \\ \midrule
		\textbf{A1.} $\text{Gen}(\cdot)$ + Sup. Condition                        & \textbf{93.4}                        & \textbf{70.9}                            & \textbf{92.9}          & {45.9}                 & {92.5}        & \textbf{89.6} & {71.1}        & {63.9}        \\
		\textbf{A2.} $\text{Gen}(\cdot)$ + Rand. Condition                       & {34.2}                               & {10.4}                                   & {30.1}                 & {7.3}                  & {10.5}        & {16.9}        & {13.9}        & {10.0}        \\
		\textbf{A3.} $\text{Gen}(\cdot)$ + $\text{Enc}(\cdot|\theta)$  (DiffAug) & \textbf{93.4}                        & {69.9}                                   & {92.5}                 & \textbf{49.7}          & \textbf{92.7} & {88.3}        & \textbf{71.8} & \textbf{64.7} \\
		\toprule
	\end{tabular}
	\label{tb_ablationa}
\end{table*}

\begin{figure*}[t]
    \centering
    \includegraphics[width=0.99\linewidth]{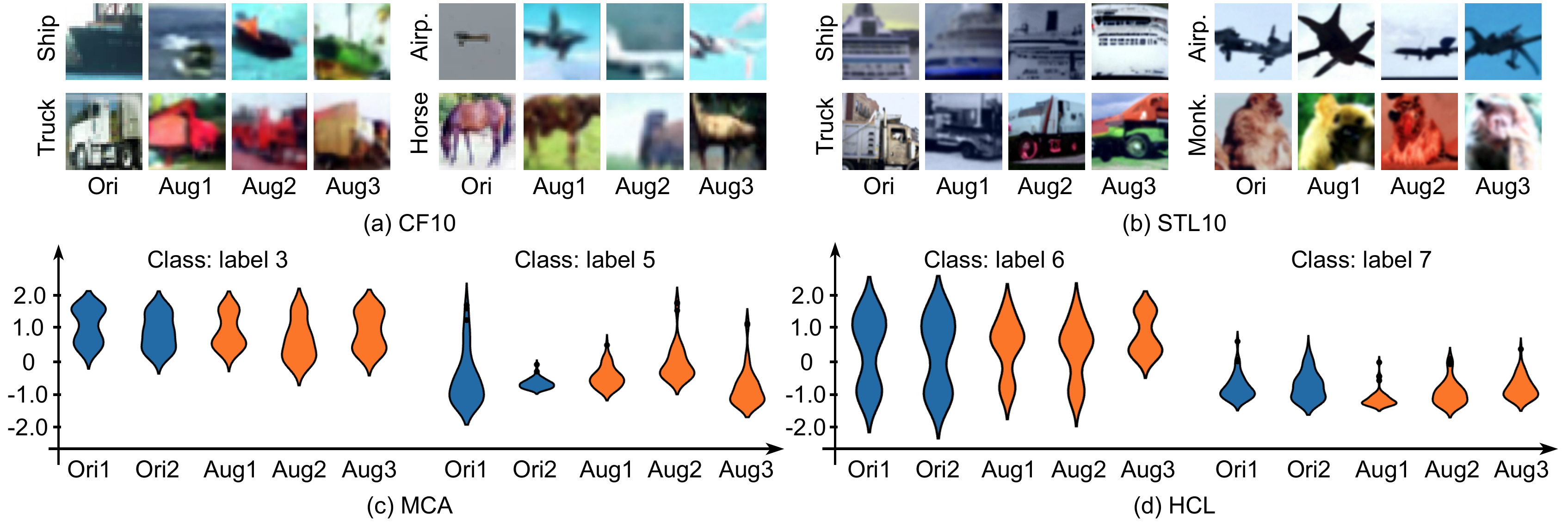}
    \caption{
        \textbf{The display of original and generated images illustrates that DiffAug generates semantically similar augmented data.} The `Ori' means original data and Aug1, Aug2 and Aug3 are augmentated data. For bio-feature data, we use violin plots to plot the distribution of features.
    }
    \label{fig_vis_image}
\end{figure*}

\subsection{Comparations on Bio-feature Datasets}\label{sub_easy}
Next, we benchmark DiffAug against SOTA unsupervised contrastive learning~(CL) methods and traditional dimensional reduction~(DR) methods in bio-feature datasets. Unlike DNA sequence data, bio-feature datasets are the format for most proteomics\cite{suhre2021genetics}, genomics~\cite{bustamante2011genomics}, and transcriptomics\cite{aldridge2020single} data. The dataset contains an equal-length vector per sample, and each element in the vector represents a gene, protein abundance, or bio-indicator. The data of training time consumption is in the Table~\ref{app_param_time_bio}.

\begin{table*}[t]
	\caption{\textbf{Ablation study of scl loss function and training strategy.} The classifier accuracy of each setting is displayed in this table. Soft contrastive learning is improved with typical contrast learning, and AB-Step training is more stable.}
	\centering
	\small
	\begin{tabular}{l||C{1.1cm}C{1.1cm}C{1.1cm}C{1.1cm}|C{1.1cm}C{1.1cm}C{1.1cm}C{1.1cm}}
		\toprule
		\multirow{2}{*}{Datasets}                         & \multicolumn{4}{c|}{Vision Datasets} & \multicolumn{4}{c}{Bio-feature Datasets}                                                                                                                   \\
		                                                  & \multirow{1}{*}{CF10}                & \multirow{1}{*}{CF100}                   & \multirow{1}{*}{STL10} & \multirow{1}{*}{TINet} & GA1457        & SAM561        & MC1374        & HCL500        \\ \midrule
		\textbf{B1.} SimCLR                               & {89.6}                               & {60.3}                                   & {89.0}                 & {45.2}                 & {84.8}        & {82.4}        & {62.3}        & {61.7}        \\
		\textbf{B2.} DiffAug w/o $\mathcal{L}_\text{df}$  & {91.3}                               & {66.1}                                   & {90.1}                 & {44.9}                 & {89.1}        & {82.1}        & {59.3}        & {62.3}        \\
		\textbf{B3.} DiffAug w/o $\mathcal{L}_\text{scl}$ & {92.7}                               & {68.4}                                   & {90.9}                 & {45.1}                 & {89.2}        & {82.4}        & {69.2}        & {61.3}        \\ \midrule
		\textbf{B4.} DiffAug Syn. Training                & {92.9}                               & {69.7}                                   & \textbf{92.7}          & {45.3}                 & {90.1}        & \textbf{89.6} & {68.1}        & {62.3}        \\
		\textbf{B5.} DiffAug AB Training                  & \textbf{93.4}                        & \textbf{69.9}                            & {92.5}                 & \textbf{49.7}          & \textbf{92.7} & {88.3}        & \textbf{71.8} & \textbf{64.7} \\
		\toprule
	\end{tabular}
	\label{tb_ablationb}
\end{table*}

\begin{figure*}
    \centering
    \includegraphics[width=0.99\linewidth]{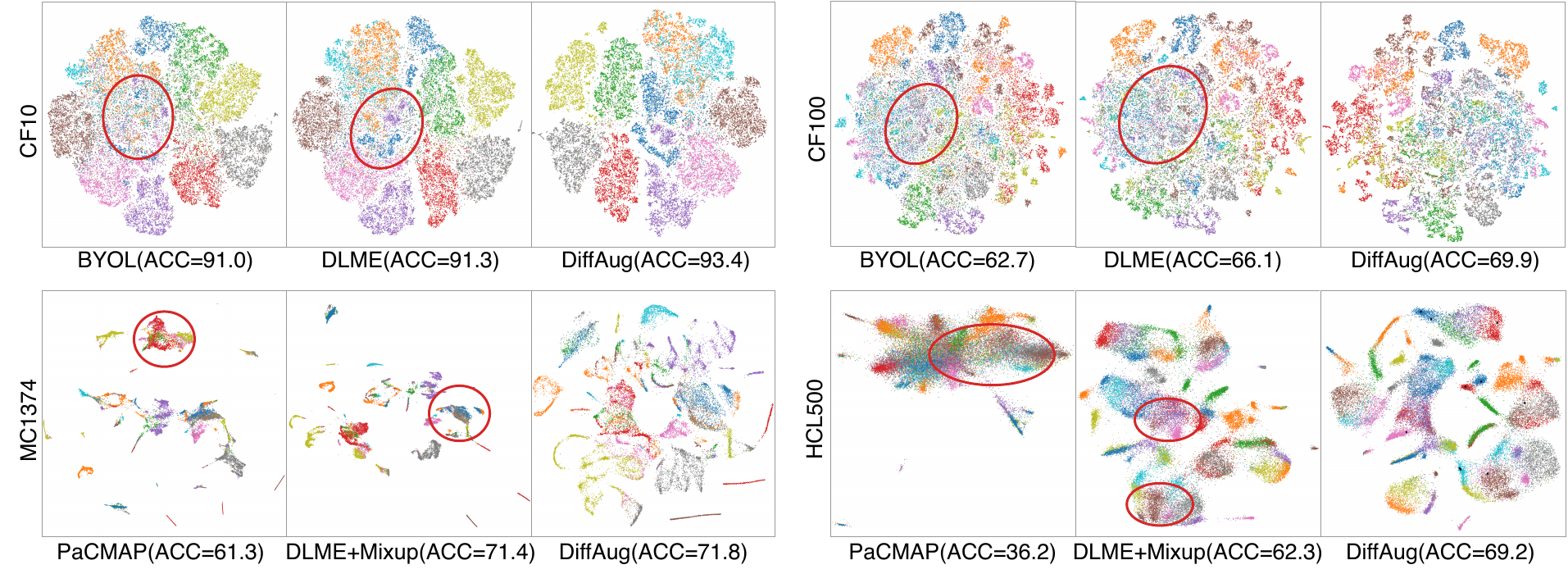}
    \caption{\textbf{The scatter visualization of representation indicates DiffAug's encoder learns cleaner embedding.} The colors represent different categories; there are 100 categories in CF100; we used the superclasses label provided by \cite{deng2021flattening}.}
    \label{fig_vis_main}
\end{figure*}

\textbf{Test protocols.}
Our experiments are conducted using a variety of bio-feature datasets, namely GA1457 \citep{rouillard2016harmonizome}, SAM~\citep{weber2016comparison}, MC1374~\citep{han2018mapping}, HCL500~\citep{han2020construction}, PROT579~\cite{sun2022artificial}, and WARPARIOP. To evaluate the effectiveness of our DiffAug, we adopted a linear Support Vector Machine (SVM) performance assessment akin to the methodologies described in \citet{Yingfan_pacmap} and \citet{sarfraz2022hierarchical}. In this assessment, dataset embeddings are split, allocating 90\% for training purposes and the remaining 10\% for testing. The comparative results are shown in Table~\ref{tab_exp_scrna}. Further details regarding this experimental setup can be found in the Appendix \ref{app_bio_exp}.
The training regimen for DiffAug is structured as follows: an initial A-Step of 330 epochs, followed by an B-Step of 330 epochs and concluding with a final A-Step of 340 epochs. Comprehensive training specifics, along with the evolution of accuracy throughout training, are depicted in Figure \ref{app_train_acc}.

\textbf{Analysis.}
{DiffAug} consistently surpasses all other methods across eight evaluations spanning four datasets, registering a performance enhancement between 0.4\% and 7.0\% over its counterparts. 
(a) It is worth noting that the benefits of DiffAug are not limited to DNA sequence data. It also excels in areas such as bio-feature and has broader applications in traditional bioanalysis.
(b) Data processed through {DiffAug} exhibits reduced overlap among distinct groups, facilitating enhanced classification. This suggests {DiffAug} delineates more explicit boundaries between data categories, culminating in more precise outcomes. Corresponding evidence is demonstrated in Figure \ref{fig_vis_main}.
(c) Experiments with DNA sequence and bio-feature data have demonstrated that DiffAug is versatile and an important complement to other unsupervised learning techniques. Searching for effective data augmentation strategies in biology has been difficult.

\subsection{Comparations on {Vision Datasets}}
\label{sub_cv}
Next, we benchmark DiffAug against the SOTA unsupervised comparative learning (CL) method on a visual dataset. This field is much more active and has seen the emergence of excellent human-designed data augmentation methods.
We focus on data augmentation techniques that can be used for unsupervised CL. Therefore, the comparison does not include some labeling-based methods~\cite {nichol2021glide,he2022synthetic,trabucco2023effective,zhang2023expanding}.

\begin{figure}[t]
    \centering
    \includegraphics[width=0.99\linewidth]{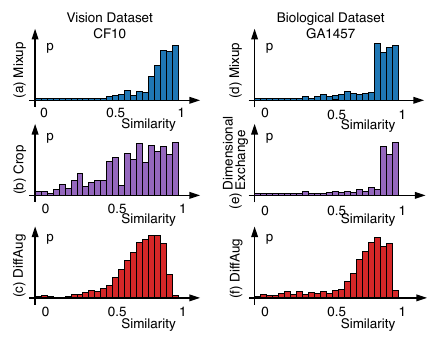}
    \caption{\textbf{Hist plot of the cosine similarity between original data and the augmentation data in latent space indicate that DiffAug generates semantically smooth augmentations.} For the image data, we compared similar mixups with random cropping. For bio-feature datasets, we compared same-label Mixup and random dimension swapping.}
    \label{fig_diff_simi_hist}
\end{figure}

\textbf{Test Protocols.} 
Experiments are performed on CIFAR-10~[CF10] and CIFAR-100~[CF100]~\citep{krizhevsky2009learning}, STL10~\citep{coates2011analysis}, TinyImageNet~[TINet]~\citep{le2015tiny} dataset. Notably, we did not use a larger dataset in our experiments. It is because the proposed method aims to efficiently utilize data to train models when data is limited (or expensive). Simple augmentation techniques can train models without fully adequate and well-sampled data.
We followed a procedure similar to SimCLR~\citep{chen_simple_2020} for the Linear-test performance assessment. The baseline of SimCLR~\citep{chen_simple_2020}, BYOL~\citep{grill_bootstrap_2020}, MoCo v2~\citep{he2020momentum}, and SimSiam~\citep{chen2021exploring} is from \citet{peng2022crafting}.
The results of SimCLR and MoCoV2 with Mixup data augmentation~(SimC.+Mix. and Mo.V2+Mix.) are from  \citet{zhang2022m}. Since we did not find the corresponding GAN and VAE as a baseline for data augmentation, we tested the corresponding results ourselves. For DiffAug, its semantic encoder served as the CL backbone, trained using DiffAug-augmented images. Comparative results are shown in Table~\ref{tab_exp_image}. Details of the experimental setup are in Appendix~\ref{app_vision_exp}. The training strategy of DiffAug is A-Step: 200 epochs $\to$ B-Step: 400 epoch $\to$ A-Step: 800 epoch. The data of training time consumption is in the Table~\ref{app_param_time_vis}.

\textbf{Analysis.} 
From Table~\ref{tab_exp_image}, it is evident that DiffAug consistently outperforms SOTA methods across all datasets. It surpasses other techniques by at least 0.8\% in five out of the four projects. This showcases the effectiveness of DiffAug's data augmentation.
(a) \textit{Beyond hand-designed augmentation methods.} DiffAug's versatility indicates that its approach is on par with, or even better than, traditional hand-crafted methods. The encoder in DiffAug produces robust features.
(b) \textit{Beyond Mixup improved CL methods.} DiffAug outperforms the Mixup improved CL method of typical contrast learning methods, and additionally, models trained using DiffAug-generated data and contrast learning methods bring some improvement.
(c) For datasets with many classes, like CF100 and TINet, DiffAug's encoder might only sometimes capture some local detail. Still, augmented data is crucial in guiding CL to produce better results.

\subsection{DiffAug Effectiveness Analysis}

\textbf{Effectiveness analysis of diffusion generator.} 
The diffusion module generates new positive data by inputting the provided condition vector. To demonstrate that our DiffAug works appropriately, we show the generation results for both the image and bio-feature datasets~(in Figure.~\ref{fig_vis_image}). A more detailed implementation and more results are in the Appendix~\ref{app_vision_exp} and Appendix~\ref{app_bio_exp}. We can observe that the generated data retains semantic similarity to the original data. For example, the objects described in the image data are consistent, while the gene distribution is also consistent. At the same time, the generated data is not simply copied but varied without changing the semantic information.

In addition, we computed the cos-similarity of the original augmented sample in latent space to explore further the semantic differences between the newly generated and original data. As depicted in Figure.~\ref{fig_diff_simi_hist}, DiffAug's similarity distribution is smoother and broader. In comparison, Mixup tends to produce augmentations that are very similar semantically, while methods like cropping might introduce data with semantically distinct noise samples. In addition, we computed the cos-similarity of the original augmented sample in latent space to explore further the semantic differences between the newly generated and original data. As depicted in Figure.~\ref{fig_diff_simi_hist}, DiffAug's similarity distribution is smoother and broader. In comparison, Mixup tends to produce augmentations that are very similar semantically, while methods like cropping might introduce data with semantically distinct noise samples.

\textbf{Effectiveness analysis of semantic encoder.}
Next, we confirm that the semantic encoder of DiffAug works well by visualizing the representation of DiffAug and baseline methods (in Figure.~\ref{fig_vis_main}). The t-SNE~\citep{van2008visualizing} is used to analyze the BYOL, DLME, and DiffAug embedding on CF10, CF100, MC1374, and HCL500 datasets. The results show that DiffAug's encoder learns cleaner embedding than the baseline methods in Figure.~\ref{fig_vis_main}, DiffAug E means the first A-Step's results of the DiffAug. By comparing DiffAug E and DiffAug, we observe that the augmented data further improves the embedding quality, significantly enhancing the depiction of local structures. The same conclusion is shown in Figure.~\ref{app_train_acc}.

\subsection{Ablation Study and Effectiveness of Component}

\textbf{Ablation study of the semantic encoder.}
In the ablation study presented in Table.~\ref{tb_ablationa}, we consider three configurations:
{A1} and {A2} confirm the significance of DiffAug's semantic encoder by ablating it in two ways. {A1} directly uses supervised one hot label as the conditional, bypassing the condition vectors generated by the unsupervised neural network. {A2} employs random conditional vectors instead of those the encoder produces.
{A3} means the proposed DiffAug method. The results from these experiments can be found in Table~\ref{tb_ablationa}.
We observe that the average performance of {A1} is highest due to the access to the label. And not accessing the label at all brings a huge performance drop.
The results in {A3} illustrate that DiffAug's performance is comparable to the fully supervised condition, demonstrating its ability to model supervised annotation within an unsupervised framework.

\textbf{Ablation study of training strategy and scl loss function.}
For Ablation in Table.~\ref{tb_ablationb},
{B1} means that the model is trained by SimCLR~\citep{chen_simple_2020}. {B2} omits the diffusion loss and trains the encoder with only the soft CL loss. {B3} omits the soft CL loss and trains the encoder with InfoNCE loss. {B4} and {B5} talk about the training strategy of DiffAug. B4 denotes training the model by integrating two loss functions, i.e., mixing A-Step and B-Step to update the parameters of both networks simultaneously through a single forward propagation.B5 denotes the default training strategy, which trains the model by alternating the two loss functions. The results from these experiments can be found in Table~\ref{tb_ablationb}.
First, we observe that either replacing the scl loss or replacing the diff model~(B2 or B3) brings about performance degradation, which implies that the two modules of DiffAug work in conjunction with each other.
Second, we observe that on some datasets, the performance of the two training strategies~(B4 and B5) is comparable, but on others, the EM method demonstrates higher stability. We attribute this to the fact that the difficulty of diffusion model training varies from data to data, and simultaneous training may result in the two modules being unable to match at all times, bringing about instability in training. However, the E-M training approach avoids this problem.

\subsection{Hyperparametric Analysis and Toxicity Analysis} \label{lambda}

Finally, we investigate the performance improvement and potential toxicity of the DiffAug method through hyperparametric analysis. The hyperparameter $\lambda$ determines how often the model generated by DiffAug affects the training of the semantic encoder. Introducing the augmentation data~($\lambda=0$) brings the method back to traditional CL methods, while too much~($\lambda=1$) will lead to the encoder crashing. To demonstrate this, we tested the model performance of different $\lambda$ counterparts on two visual datasets (CF10, CF100) and two bio-feature datasets (SAM561 and MC1374).
As shown in Figure.~\ref{fig_vis_image}, the change in performance brought about by $\lambda$ is consistent across datasets. Specifically, setting $\lambda=0.1$ or $\lambda=0.15$ provides the most significant gain. We believe that $\lambda=0.1$ may be a suitable default setting for most datasets.

\section{Conclusion}
In summary, we presented DiffAug, an innovative contrastive learning framework that leverages diffusion-based augmentation to enhance the robustness and generalization of unsupervised learning. Unlike many existing methods, DiffAug operates independently of prior knowledge or external labels, positioning it as a versatile augmentation tool with notable performance in vision and life sciences. Our tests reveal that DiffAug consistently boosts classification and clustering accuracy across multiple datasets.

\section*{Acknowledgements}
This work was supported by the National Science and Technology Major Project (No.2022ZD0115101), the National Natural Science Foundation of China Project (No. U21A20427), and Project (No.WU2022A009) from the Center of Synthetic Biology and Integrated Bioengineering of Westlake University. We thank the Westlake University HPC Center for providing computational resources. This work was done when Zelin Zang was intern at DAMO Academy. This work was supported by Alibaba Group through Alibaba Research Intern Program.

This research is supported by the National Research Foundation, Singapore under its AI Singapore Programme (AISG Award No: AISG2-PhD-2021-08-008). Yang You's research group is being sponsored by NUS startup grant (Presidential Young Professorship), Singapore MOE Tier-1 grant, ByteDance grant, ARCTIC grant, SMI grant (WBS number: A-8001104-00-00),  Alibaba grant, and Google grant for TPU usage.

\section*{Impact Statements}
The introduction of DiffAug, our novel approach to unsupervised contrastive learning, presents an opportunity to reflect on the ethical considerations inherent in the advancement of machine learning technologies. A key aspect of DiffAug is its reliance on a conditional diffusion model for the generation of new, positive data samples. This process, rooted in unsupervised learning, underscores our commitment to minimizing human intervention and, by extension, the potential for human bias in the initial stages of data handling. By automating the generation of training data, we aim to reduce the subjective influences that may arise from hand-designed methods, thus promoting a more objective and equitable development of machine learning models. Furthermore, the iterative training process of the semantic encoder and diffusion model in DiffAug is designed to ensure that the generated data remains true to the original dataset's semantic integrity, thereby upholding the principles of fairness and transparency in AI. 

The deployment of DiffAug is poised to have a transformative effect on a wide array of sectors, leveraging the untapped potential of unsupervised learning to interpret complex datasets across disciplines. In healthcare, for example, DiffAug's ability to enhance representation learning without extensive labeled datasets could revolutionize the early detection and diagnosis of diseases, making healthcare more accessible and efficient. The fundamental shift towards unsupervised learning, exemplified by DiffAug, also heralds a new era of innovation in which the reliance on large, labeled datasets is significantly reduced, thereby democratizing access to advanced machine learning technologies.

\bibliography{MyLibrary,paper}
\bibliographystyle{icml2024}

\newpage
\appendix
\onecolumn

\begin{center}
	{\Large
		Appendix}
	\vspace{5mm}
\end{center}

\tableofcontents

\section{Appendix: Related Works} \label{app_releadwork}
\paragraph{Generative Models } Generative models have been the subject of growing interest and rapid advancement. Earlier methods, including VAEs \citep{VAE} and GANs \citep{GAN}, showed initial promise generating realistic images, and were scaled up in terms of resolution and sample quality \citep{BigGAN, VQ-VAE-2}. Despite the power of these methods, many recent successes in photorealistic image generation were the result of diffusion models \citep{DDPM,ImprovedDDPM,Imagen,GLIDE,unCLIP}. Diffusion models have been shown to generate higher-quality samples compared to their GAN counterparts \citep{Dhariwal21}, and developments like classifier free guidance \citep{ClassifierFreeGuidance} have made text-to-image generation possible. Recent emphasis has been on training these models with internet-scale datasets like LAION-5B \citep{LAION5B}. Generative models trained at internet-scale \citep{StableDiffusion,Imagen,GLIDE,unCLIP} have unlocked several application areas where photorealistic generation is crucial.

\paragraph{Synthetic Image Data Generation } Training neural networks on synthetic data from generative models was popularized using GANs \citep{DAGAN,Tran17,Zheng17}. Various applications for synthetic data generated from GANs have been studied, including representation learning \citep{Jahanian22}, inverse graphics \citep{StyleGANRender}, semantic segmentation \citep{DatasetGAN}, and training classifiers \citep{Tanaka19,Dat19,Yamaguchi20,Besnier20,Xiong_2020_CVPR,Sajila2021,Haque21}. More recently, synthetic data from diffusion models has also been studied in a few-shot setting \citep{RealGuidance}. These works use generative models that have likely seen images of target classes and, to the best of our knowledge, we present the first analysis for synthetic data on previously unseen concepts. 
\cite{du2023dream} proposed the DREAM-OOD framework, which uses diffusion models to generate photo-realistic outliers from in-distribution data for improved OOD detection. By learning a text-conditioned latent space, it visualizes imagined outliers directly in pixel space, showing promising results in empirical studies. \cite{zhang2023expanding} developed the Guided Imagination Framework (GIF) using generative models like DALL-E2 and Stable Diffusion for dataset expansion, enhancing accuracy in both natural and medical image datasets.

\paragraph{Synthetic Biology Data Generation } The realm of synthetic biology has witnessed a surge in the utilization of data-driven approaches, particularly with the advent of advanced computational models. The generation of synthetic biological data has been instrumental in predicting protein structures \citep{mcgibbon2023scorch}. The use of Generative Adversarial Networks (GANs) has also found its way into this domain, aiding in the c
reation of synthetic DNA sequences \citep{zheng2023maskdna,li2022phiaf,han2019progan} and simulating cell behaviors \citep{botton2022data}. Furthermore, the integration of machine learning with synthetic biology has paved the way for innovative solutions in drug discovery \citep{blanco2023role,mcgibbon2023scorch}. Unlike the synthetic image data generation, where models have often seen images of target classes, synthetic biology data generation often grapples with the challenge of generating data for entirely novel biological entities. This presents a unique set of challenges and opportunities, pushing the boundaries of what synthetic data can achieve in the realm of biology.

\section{Appendix: Details of Contrastive Learning and Soft Contrastive Learning} \label{app_scl}
\subsection{The t-kernel similarity in soft contrastive learning}
To map the high-dimensional embedding vector to a probability value, a kernel function $\mathcal{S}(\cdot)$ is used. In this paper, we use the t-distribution kernel function $\mathcal{S}^\nu(\cdot)$ because it exposes the degrees of freedom and allows us to adjust the closeness of the distribution in the dimensionality reduction mapping~\citep{BolianLi2021TrustworthyLC}. The t-distribution kernel function is defined as follows,
\begin{equation}
	\begin{aligned}
		 & \mathcal{S}(\mathbf{z}_i, \mathbf{z}_j) =
		{\Gamma\left((\nu +1)/{2}\right)}
		\left(
		1+{\|z_i- z_j\|_{2}^{2}}/{\nu}
		\right)^{-\frac{\nu +1}{2}} /{\sqrt{\nu  \pi} \Gamma\left({\nu }/{2}\right)},
	\end{aligned}
	\label{eq_tkernal}
\end{equation}
where $\Gamma(\cdot)$ is the Gamma function. The degrees of freedom $\nu$ control the shape of the kernel function. The different degrees of freedom~($\nu^y, \nu^z$) is used in $\mathcal{R}^y$ and $\mathcal{R}^z$ for the dimensional reduction mapping.

\subsection{Why Soft Contrastive Learning is a softened version of Contrastive Learning}

\begin{lemma}
	Let $\mathcal{L}_{\mathrm{cl}}=-\log {\mathcal{Q} \left(\mathbf{z}_{i} ,\mathbf{z}_{i}^{+}\right)} + \log\left[{\mathcal{Q} \left(\mathbf{z}_{i}, \mathbf{z}_{i}^{+}\right)+\sum_{\mathbf{z}_{i}^{-} \in V^{-}} \mathcal{Q} \left(\mathbf{z}_{i}, \mathbf{z}_{i}^{-}\right)}\right] $
	and
	$\mathcal{L}_{\mathrm{cl}}^p = -
		\sum_{j=1}^{N_K+1} \left\{  \mathcal{H}_{ij} \log Q_{ij} +  (1-\mathcal{H}_{ij}) \log \dot{Q}_{ij}
		\right\}
	$ \text{Then} $\lim_{x \to \infty} \mathcal{L}_{\mathrm{cl}} - \mathcal{L}^p_{\mathrm{cl}} =0$  .
\end{lemma}

\begin{proof}

	We start with $L_\text{CL} =  - \log \frac{ \exp(S(z_i, z_j))}{\sum_{k=1}^{N_K} \exp(S(z_i, z_k))}$ (Eq.~(\ref{eq_CL})), then

	$$ L_\text{CL} = \log N_K - \log \frac{ \exp(S(z_i, z_j))}{\frac{1}{N_K}\sum_{k=1}^{N_K} \exp(S(z_i, z_k))}. $$

	We are only concerned with the second term that has the gradient. Let $(i,j)$ are positive pair and $(i,k_1), \cdots, (i,k_N) $ are negative pairs. The overall loss associated with point $i$ is:
	\begin{equation*}
		\begin{aligned}
			  & - \log \frac
			{\exp(S(z_i, z_j))}
			{ \frac{1}{N_K} \sum_{k=1}^{N_K} \exp(S(z_i, z_k))}                                                          \\
			= & - \left[
				\log \exp(S(z_i, z_j)) - \log
			{ \frac{1}{N_K} \sum_{k=1}^{N_K} \exp(S(z_i, z_k))} \right]                                                  \\
			= & - \left[
				\log \exp(S(z_i, z_j)) - \sum_{k=1}^{N_K} \log \exp(S(z_i, z_{k})) +
			\sum_{k=1}^{N_K} \log \exp(S(z_i, z_{k})) - \log { \frac{1}{N_K} \sum_{k=1}^{N_K} \exp(S(z_i, z_k))} \right] \\
			= & - \left[
				\log \exp(S(z_i, z_j)) - \sum_{k=1}^{N_K} \log \exp(S(z_i, z_{k})) +
			\log \Pi_{k=1}^{N_K}  \exp(S(z_i, z_{k})) - \log { \frac{1}{N_K} \sum_{k=1}^{N_K} \exp(S(z_i, z_k))} \right] \\
			= & - \left[
				\log \exp(S(z_i, z_j)) - \sum_{k=1}^{N_K} \log \exp(S(z_i, z_{k})) +
			\log \frac {\Pi_{k=1}^{N_K}  \exp(S(z_i, z_{k}))}{ \frac{1}{N_K} \sum_{k=1}^{N_K} \exp(S(z_i, z_k))} \right] \\
		\end{aligned}
	\end{equation*}

	We focus on the case where the similarity is normalized, $S(z_i, z_k) \in [0,1]$. The data $i$ and data $k$ is the negative samples, then $S(z_i, z_k)$ is near to $0$, $\exp(S(z_i, z_{k}))$ is near to $1$, thus the $\frac {\Pi_{k=1}^{N_K} \exp(S(z_i, z_{k}))}{ \frac{1}{N} \sum_{k=1}^{N_K} \exp(S(z_i, z_k))}$ is near to 1, and $\log \frac {\Pi_{k=1}^{N_K}  \exp(S(z_i, z_{k}))}{ \frac{1}{N} \sum_{k=1}^{N_K} \exp(S(z_i, z_k))}$ near to 0. We have

	\begin{equation*}
		\begin{aligned}
			L_\text{CL}
			 & \approx  - \left[
			\log \exp(S(z_i, z_j)) - \sum_{k=1}^{N_K} \log \exp(S(z_i, z_{k})) \right] \\
		\end{aligned}
	\end{equation*}

	We denote $ij$ and $ik$ by a uniform index and use $\mathcal{H}_{ij}$ to denote the homology relation of $ij$.

	\begin{equation*}
		\begin{aligned}
			L_\text{CL}
			 & \approx - \left[
			\log \exp(S(z_i, z_j)) - \sum_{k=1}^{N_K} \log \exp(S(z_i, z_{k})) \right]                                       \\
			 & \approx - \left[
			\mathcal{H}_{ij} \log \exp(S(z_i, z_j)) - \sum_{j=1}^{N_K} (1-\mathcal{H}_{ij}) \log \exp(S(z_i, z_{j})) \right] \\
			 & \approx - \left[
				\sum_{j=1}^{N_K+1} \left\{  \mathcal{H}_{ij} \log \exp(S(z_i, z_j)) +  (1-\mathcal{H}_{ij}) \log \{\exp(-S(z_i, z_{j}))\}
				\right\}
			\right]                                                                                                          \\
		\end{aligned}
	\end{equation*}

	we define the similarity of data $i$ and data $j$ as $Q_{ij} = \exp(S(z_i, z_j))$ and the dissimilarity of data $i$ and data $j$ as $\dot{Q}_{ij} =  \exp(-S(z_i, z_j))$.

	\begin{equation*}
		\begin{aligned}
			L_\text{CL} \approx - \left[
				\sum_{j=1}^{N_K+1} \left\{  \mathcal{H}_{ij} \log Q_{ij} +  (1-\mathcal{H}_{ij}) \log \dot{Q}_{ij}
				\right\}
			\right] \\
		\end{aligned}
	\end{equation*}
\end{proof}

\textbf{The proposed SCL loss is a smoother CL loss:}

This proof tries to indicate that the proposed SCL loss is a smoother CL loss. We discuss the differences by comparing the two losses to prove this point.
the forward propagation of the network is,
${z}_{i}=H(\hat{z}_{i}), \hat{z}_{i} =F(x_{i})$,
${z}_{j}=H(\hat{z}_{j}), \hat{z}_{j} =F(x_{j})$.
We found that we mix $y$ and $\hat{z}$ in the main text, and we will correct this in the new version. So, in this section
${z}_{i}=H(y_{i}), y_{i} =F(x_{i})$,
${z}_{j}=H(y_{j}), y_{j} =F(x_{j})$ is also correct.

Let $H(\cdot)$ satisfy $K$-Lipschitz continuity, then
$
	d^z_{ij} = k^* d^y_{ij} , k^* \in [1/K, K],
$
where $k^*$ is a Lipschitz constant. The difference between $L_\text{SCL}$ loss and $L_\text{CL}$ loss is,
\begin{equation}
	\begin{aligned}
		L_{\text{CL}}- L_\text{SCL} \approx
		\sum_j \biggl[
		\left(
		\mathcal{H}_{ij} - [1+(e^\alpha -1)\mathcal{H}_{ij}] \kappa \left(d^y_{ij} \right)
		\right)
		\log
		\left(
		\frac
		{1}
		{\kappa \left( d_{ij}^{z}\right)}
		-
		1
		\right)
		\biggl] .
	\end{aligned} \label{eq_sclcl}
\end{equation}
Because the  $\alpha > 0$, the proposed SCL loss is the soft version of the CL loss. if $\mathcal{H}_{ij}=1$, we have:

\begin{equation}
	\begin{aligned}
		(L_{\text{CL}} - L_{\text{SCL}})  |_{\mathcal{H}_{ij} =1} = \sum
		\biggl[
			\left(
			(1 - e^\alpha) \kappa \left( k^* d^z_{ij} \right)
			\right)
			\log
			\left(
			\frac
			{1}
			{\kappa \left( d_{ij}^{z}\right)}
			-
			1
			\right)
		\biggl] \\
	\end{aligned}
\end{equation}

then:

\begin{equation}
	\begin{aligned}
		\lim_{\alpha \to 0}
		( L_{\text{CL}} - L_{\text{SCL}} ) |_{\mathcal{H}_{ij} =1}
		= \lim_{\alpha \to 0} \sum
		\biggl[
			\left(
			(1 - e^\alpha) \kappa \left( k^* d^z_{ij} \right)
			\right)
			\log
			\left(
			\frac
			{1}
			{\kappa \left( d_{ij}^{z}\right)}
			-
			1
			\right)
			\biggl]              = 0
	\end{aligned}
	\label{eq:lim}
\end{equation}

Based on Eq.(\ref{eq:lim}), we find that if $i,j$ is neighbor~($\mathcal{H}_{ij}=1$) and $\alpha\to0$, there is no difference between the CL loss $L_\text{CL}$ and SCL loss $L_{\text{SCL}}$.
When if $\mathcal{H}_{ij}=0$, the difference between the loss functions will be the function of $d_{ij}^{z}$. The CL loss $L_\text{CL}$ only minimizes the distance between adjacent nodes and does not maintain any structural information. The proposed SCL loss considers the knowledge both comes from the output of the current bottleneck and data augmentation, thus less affected by view noise.

\vspace{5mm}

\textbf{Details of Eq.~(\ref{eq_sclcl}).}
Due to the very similar gradient direction, we assume $\dot{Q}_{ij} = 1-Q_{ij}$. The contrastive learning loss is written as,
\begin{equation}
	\begin{aligned}
		L_\text{CL}  \approx & - \sum
		\left\{
		\mathcal{H}_{ij}
		\log
		Q_{ij}
		+
		\left(1-\mathcal{H}_{ij} \right)
		\log
		\left(1-{Q}_{ij} \right)
		\right\}
	\end{aligned}
\end{equation}
where $\mathcal{H}_{ij}$ indicates whether $i$ and $j$ are augmented from the same original data.

The SCL loss is written as:

\begin{equation}
	\begin{aligned}
		L_{\text{SCL}} & =
		-
		\sum
		\left\{
		P_{ij}
		\log
		Q_{ij}
		+
		\left(1-P_{ij}\right)
		\log
		\left(1-Q_{ij}\right)
		\right\}
	\end{aligned}
	\label{eq:appendix_SCL}
\end{equation}

According to Eq.~(4) and Eq.~(5), we have

\begin{equation}
	\begin{aligned}
		P_{ij} & = R_{ij} \kappa(d^y_{ij}) = R_{ij} \kappa(y_i, y_j),
		R_{ij} = \left\{
		\begin{array}{lr}
			e^\alpha   \;\;\; \text{if} \;\; \mathcal{H}(x_i, x_j)=1 \\
			1  \;\;\;\;\;\;\;  \text{otherwise}                      \\
		\end{array}
		\right.,                                                      \\
		Q_{ij} & = \kappa(d_{ij}^z) = \kappa(z_i, z_j),
	\end{aligned}
\end{equation}

For ease of writing, we use distance as the independent variable, $d_{ij}^y=\|y_i- y_j\|_2$, $d_{ij}^z=\|z_i- z_j\|_2$.

The difference between the two loss functions is:

\begin{equation}
	\begin{aligned}
		  & L_\text{CL} - L_{\text{SCL}} \\
		= & -\sum\biggl[
			\mathcal{H}_{ij}
			\log \kappa \left( d_{ij}^{z}\right)
			+
			\left(1-\mathcal{H}_{ij} \right)
			\log \left(1-\kappa \left(d_{ij}^{z}\right)\right)
			-
			R_{ij}\kappa\left( d^y_{ij} \right)
			\log \kappa \left( d^z_{ij} \right)
			-
			\left(1- R_{ij}\kappa\left( d^y_{ij} \right)\right)
			\log \left(1-\kappa \left( d^z_{ij} \right)\right)
		\biggl]                          \\
		= & -\sum\biggl[
			\left(
			\mathcal{H}_{ij} -  R_{ij}\kappa\left(d^y_{ij} \right)
			\right)
			\log \kappa \left( d_{ij}^{z}\right)
			+
			\left(
			1-\mathcal{H}_{ij} -1 +  R_{ij}\kappa\left(d^y_{ij} \right)
			\right)
			\log \left(1-\kappa \left(d^z_{ij}\right)\right)
		\biggl]                          \\
		= & -\sum\biggl[
			\left(
			\mathcal{H}_{ij} - R_{ij} \kappa \left(d^y_{ij} \right)
			\right)
			\log \kappa \left( d_{ij}^{z}\right)
			+
			\left(
			R_{ij} \kappa \left(  d^y_{ij} \right) - \mathcal{H}_{ij}
			\right)
			\log \left(1-\kappa \left(d^z_{ij}\right)\right)
		\biggl]                          \\
		= & -\sum\biggl[
			\left(
			\mathcal{H}_{ij} - R_{ij}\kappa \left(d^y_{ij} \right)
			\right)
			\left(
			\log \kappa \left( d_{ij}^{z}\right)
			-
			\log \left(1-\kappa \left(d^z_{ij}\right)\right)
			\right)
		\biggl]                          \\
		= & \sum\biggl[
			\left(
			\mathcal{H}_{ij} - R_{ij} \kappa \left(d^y_{ij} \right)
			\right)
			\log
			\left(
			\frac
			{1}
			{\kappa \left( d_{ij}^{z}\right)}
			-
			1
			\right)
		\biggl]                          \\
	\end{aligned}
	\label{eq:appendix_diff_twoloss_3_0}
\end{equation}

Substituting the relationship between $\mathcal{H}_{ij}$ and $R_{ij}$, $R_{ij} = 1+(e^\alpha -1)\mathcal{H}_{ij}$, we have

\begin{equation}
	\begin{aligned}
		L_{\text{CL}} - L_{\text{SCL}}=\sum
		\biggl[
		\left(
		\mathcal{H}_{ij} - [1+(e^\alpha -1)\mathcal{H}_{ij}] \kappa \left(d^y_{ij} \right)
		\right)
		\log
		\left(
		\frac
		{1}
		{\kappa \left( d_{ij}^{z}\right)}
		-
		1
		\right)
		\biggl] \\
	\end{aligned}
	\label{eq:appendix_diff_twoloss_3_1}
\end{equation}

We assume that network $H(\cdot)$ to be a Lipschitz continuity function, then

\begin{equation}
	\begin{aligned}
		\frac{1}{K} H(d^z_{ij}) \leq d^y_{ij} \leq K H(d^z_{ij}) \quad \forall i, j \in \{1,2,\cdots,N\}
	\end{aligned}
\end{equation}

We construct the inverse mapping of $H(\cdot)$ to $H^{-1}(\cdot)$,

\begin{equation}
	\begin{aligned}
		\frac{1}{K} d^z_{ij} \leq d^y_{ij} \leq K d^z_{ij} \quad \forall i, j \in \{1,2,\cdots,N\}
	\end{aligned}
\end{equation}

and then there exists $k^*$:
\begin{equation}
	\begin{aligned}
		d^y_{ij} = k^* d^z_{ij} \quad k^* \in [1/K, K] \quad \forall i, j \in \{1,2,\cdots,N\}
	\end{aligned}
	\label{eq:g_revers}
\end{equation}

Substituting the Eq.(\ref{eq:g_revers}) into Eq.(\ref{eq:appendix_diff_twoloss_3_1}).

\begin{equation}
	\begin{aligned}
		L_{\text{CL}} - L_{\text{SCL}}=\sum
		\biggl[
		\left(
		\mathcal{H}_{ij} - [1+(e^\alpha-1)\mathcal{H}_{ij}] \kappa \left(k^* d^z_{ij}\right)
		\right)
		\log
		\left(
		\frac
		{1}
		{\kappa \left( d_{ij}^{z}\right)}
		-
		1
		\right)
		\biggl] \\
	\end{aligned}
	\label{eq:appendix_diff_twoloss_4}
\end{equation}

\section{Appendix: Details of DNA Sequence Experiments} \label{app_dna_exp}

\subsection{Experimental Setups and Datasets Information}

In our research, we utilized the 'genomic-benchmarks' dataset, a comprehensive collection of curated sequence classification datasets specifically designed for genomic studies. This repository encompasses a range of datasets derived from both novel compilations mined from publicly accessible databases and existing datasets gathered from published studies. It focuses on regulatory elements such as promoters, enhancers, and open chromatin regions from three model organisms: humans, mice, and roundworms. Accompanying these datasets is a simple convolutional neural network provided as a baseline model, enabling researchers to benchmark their algorithms effectively. The entire collection is made available as a Python package, facilitating easy integration with popular deep learning libraries and serving as a valuable resource for the genomics research community.

The initiative behind the 'genomic-benchmarks' dataset aims to address the critical need for standardized benchmarks in genomics, akin to the role that carefully curated benchmarks have played in advancing other biological fields, notably demonstrated by AlphaFold's success in protein folding. By offering a structured and easily accessible set of benchmarks, this collection not only promotes comparability and reproducibility in machine learning applications within genomics but also lowers the entry barrier for researchers new to this domain. Consequently, it fosters a healthy competitive environment that is likely to spur innovation and discovery in genomic research, paving the way for significant advancements in the understanding and annotation of genomes.

\begin{table}[h]
	\centering
	\caption{GenomicBenchmarks Dataset Metadata}
	\label{tab:dataset_metadata}
	\begin{tabular}{lccccc}
		\toprule
		\textbf{Name}                        & Acronyms & \textbf{Num. Seqs} & \textbf{Num. Classes} & \textbf{Median Len} & \textbf{Std} \\
		\midrule
		dummy\_mouse\_enhancers\_ensembl     & MoEnEn   & 1,210              & 2                     & 2,381               & 984.4        \\
		demo\_coding\_vs\_intergenomic\_seqs & CoIn     & 100,000            & 2                     & 200                 & 0            \\
		demo\_human\_or\_worm                & HuWo     & 100,000            & 2                     & 200                 & 0            \\
		human\_enhancers\_cohn               & HuEnCo   & 27,791             & 2                     & 500                 & 0            \\
		human\_enhancers\_ensembl            & HuEnEn   & 154,842            & 2                     & 269                 & 122.6        \\
		human\_ensembl\_regulatory           & HuEnRe   & 289,061            & 3                     & 401                 & 184.3        \\
		human\_nontata\_promoters            & HuNoPr   & 36,131             & 2                     & 251                 & 0            \\
		human\_ocr\_ensembl                  & HuOcEn   & 174,756            & 2                     & 315                 & 108.1        \\
		\bottomrule
	\end{tabular}
\end{table}

\section{Appendix: Details of Vision Experiments} \label{app_vision_exp}

\subsection{Dataset Setups}

Experiments are performed on CIFAR-10~[CF10]\footnote{https://www.cs.toronto.edu/~kriz/cifar.html} and CIFAR-100\footnote{https://www.cs.toronto.edu/~kriz/cifar.html}~[CF100]~\citep{krizhevsky2009learning}, STL10\footnote{https://cs.stanford.edu/~acoates/stl10/}~\citep{coates2011analysis}, TinyImageNet\footnote{https://www.kaggle.com/c/tiny-imagenet}~[TINet]~\citep{le2015tiny} dataset.

To compare with the two different baseline methods, the setting of the dataset is shown in Table.~\ref{at:6}.
\begin{table*}[h]
	\centering
	\caption{Dataset setting of linear-test Performance.}
	\begin{tabular}{cccccc}
		\hline \text { Dataset } & \text { Train Split }      & \text { Test Split } & \text { Train Samples } & \text { Test Samples } & \text { Classes } \\
		\hline \text { CF10 }    & \text { Train }            & \text { Test }       & 50,000                  & 10,000                 & 10                \\
		\text { CF100 }          & \text { Train }            & \text { Test }       & 50,000                  & 10,000                 & 100               \\
		\text { STL10 }          & \text { Train + Unlabeled} & \text { Test }       & 5,000+100,000           & 8,000                  & 10                \\
		\text { TINet }          & \text { Train }            & \text { Test }       & 100,000                 & 100,000                & 200               \\
		\hline
	\end{tabular}
	\label{at:6}
\end{table*}

\begin{table*}[h]
	\centering
	\caption{Dataset setting of clustering test.}
	\begin{tabular}{cccc}
		\hline \text { Dataset } & \text { Train \& Test Split } & \text { Train \& Test Samples } & \text { Classes } \\
		\hline \text { CF10 }    & \text { Train+Test }          & 60,000                          & 10                \\
		\text { CF100 }          & \text { Train+Test }          & 60,000                          & 20                \\
		\text { STL10 }          & \text { Train+Test }          & 13,000                          & 10                \\
		\text { TIN }            & \text { Train }               & 100,000                         & 200               \\
		\hline
	\end{tabular}
	\label{at:8}
\end{table*}

\subsection{Baseline Methods and Implementation Details}

The contrastive learning methods, including SimCLR~\citep{chen_simple_2020}, MOCO v2~\citep{he_momentum_2020}, BYOL~\citep{grill_bootstrap_2020}, SimSiam~\citep{chen2021exploring}, and DLME~\citep{zang2022dlme} are chosen for comparison. The SimC.+Mix. and MoCo.+Mix. are SimCLR and MoCoV2 with Mixup data augmentation which processed by \cite{zhang2022m}. The SimC.+Dif. and MoCo.+Mix. are SimCLR and MoCoV2 with DiffAug data augmentation. Improvements over the best baseline are shown in parentheses.

For the Linear-test performance assessment, we followed a procedure similar to SimCLR~\citep{chen_simple_2020}. We evaluated the model's representations linearly on top of the frozen features. This ensures that the quality of the representations is attributed only to the pre-training task, without any influence from potential fine-tuning. We used the ResNet-50~\citep{he_deep_2015} backbone as the encoder and a standard diffusion backbone as diffusion model~(in code below).
In contrast, for DiffAug, its semantic encoder served as the contrastive learning backbone, trained using DiffAug-augmented images. For the kMeans clustering evaluation, we extracted feature vectors from the models, leaving out the top classification layer. We then applied kMeans clustering to these features. The primary metric for evaluation was clustering accuracy.

\begin{lstlisting}[language=Python, caption=DiffusionModel for Vision Task]
class DiffusionModelVision(nn.Module):
    def __init__(self, c_in=3, c_out=3, time_dim=256):
        super().__init__()
        self.time_dim = time_dim
        self.remove_deep_conv = remove_deep_conv
        self.inc = DoubleConv(c_in, 16)
        self.down1 = Down(16, 32)
        self.sa1 = SelfAttention(32)
        self.down2 = Down(32, 64)
        self.sa2 = SelfAttention(64)
        self.down3 = Down(64, 64)
        self.sa3 = SelfAttention(64)
        self.up1 = Up(128, 32)
        self.sa4 = SelfAttention(32)
        self.up2 = Up(64, 16)
        self.sa5 = SelfAttention(16)
        self.up3 = Up(32, 16)
        self.sa6 = SelfAttention(16)
        self.outc = nn.Conv2d(16, c_out, kernel_size=1)
        def pos_encoding(self, t, channels):
        inv_freq = 1.0 / (10000 ** (torch.arange(0, channels, 2, device=one_param(self).device).float() / channels))
        pos_enc_a = torch.sin(t.repeat(1, channels // 2) * inv_freq)
        pos_enc_b = torch.cos(t.repeat(1, channels // 2) * inv_freq)
        pos_enc = torch.cat([pos_enc_a, pos_enc_b], dim=-1)
        return pos_enc
    
    def forward(self, x, t):
        t = t.unsqueeze(-1)
        t = self.pos_encoding(t, self.time_dim)
        return self.unet_forwad(x, t)
\end{lstlisting}

Our training strategy is as follows: A-step: 200 epochs $\to$ B-Step: 400 epoch $\to$ A-Step: 800 epoch. Continued training will further improve performance, but we did not increase the amount of computation due to computational resource constraints. The time loss of the method does improve due to the use of the diffusion model. However, on small datasets, this boost is acceptable. In this way at the same time DiffAug gives the possibility to accomplish unsupervised comparison learning training on small datasets.

\begin{table}[!htbp]
	\begin{center}
		\caption{Details of the training process in vision dataset.}
		\label{app_param_time_vis}
		\begin{tabular}{lcccccccc}
			\hline
			CF10   & $\nu$ & Learning Rate & Weight Decay & Batch Size & GPU    & pix            & Training Time & \\ \hline
			CF10   & 1     & 0.001         & 1e-6         & 256        & 1*V100 & 32$\times$32   & 7.1 hours     & \\
			CF100  & 2     & 0.001         & 1e-6         & 256        & 1*V100 & 32$\times$32   & 7.2 hours     & \\
			STL10  & 5     & 0.001         & 1e-6         & 256        & 1*V100 & 96$\times$96   & 15.1 hours    & \\
			TINet  & 3     & 0.001         & 1e-6         & 256        & 1*V100 & 64$\times$64   & 20.6 hours    & \\ \hline
		\end{tabular}
	\end{center}
\end{table}

\begin{figure}
	\centering
	\includegraphics[width=0.99\linewidth]{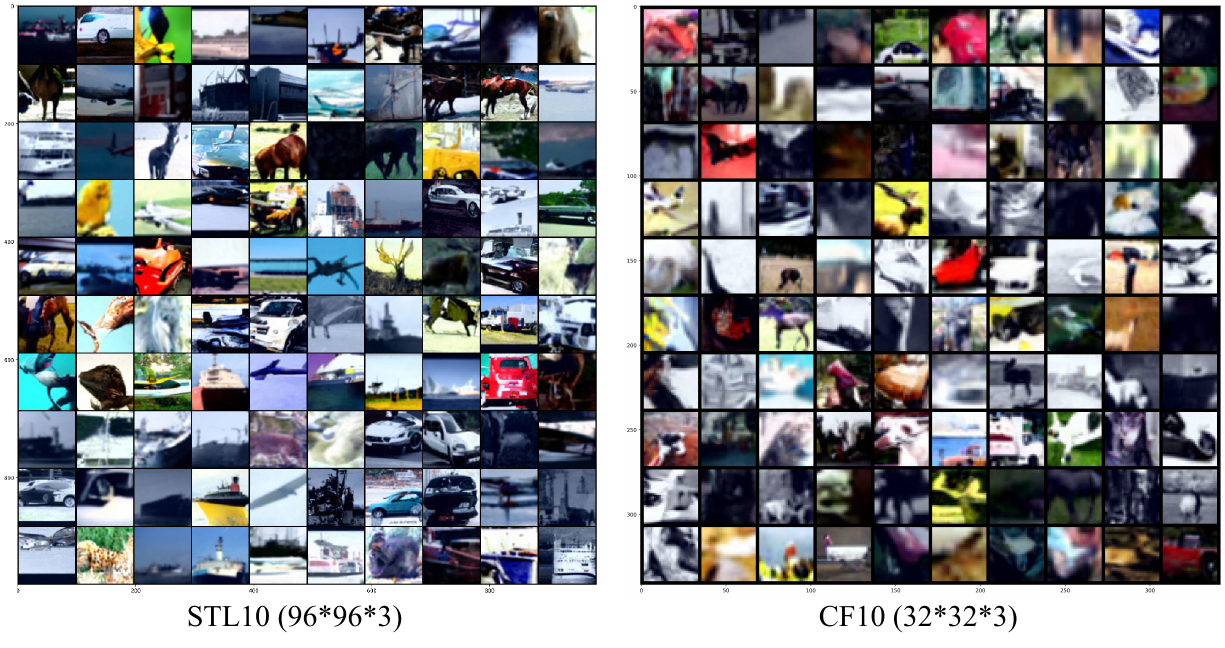}
	\caption{\textbf{The display of original and generated images illustrates that DiffAug generates semantically similar augmented images.} Ori means original image and Aug1, Aug2 and Aug3 are augmentated images. More detailed results are in the appendix.}
	\label{app_vis_main}
\end{figure}

\subsection{Data Augmentation of the Compared Methods} \label{app_datasug}

\paragraph{BYOL augmentation.}
The BYOL augmentation method is a hand-designed method. It is composed of four parts: random cropping, left-right flip, color ji
\begin{itemize}
	\item Random cropping: A random patch of the image is selected, with an area uniformly sampled between 8\% and 100\% of that of the original image, and an aspect ratio logarithmically sampled between 3/4 and 4/3. This patch is then resized to the target size of 224 $\times$ 224 using bicubic interpolation.
	\item Optional left-right flip.
	\item Color jittering: The brightness, contrast, saturation, and hue of the image are shifted by a uniformly random offset applied to all the pixels of the same image. The order in which these shifts are performed is randomly selected for each patch.
	\item Color dropping: An optional conversion to grayscale. When applied, the output intensity for a pixel (r, g, b) corresponds to its luma component, computed as 0.2989r + 0.5870g + 0.1140b1.
\end{itemize}

\paragraph{SimCLR augmentation.}
\begin{itemize}
	\item Random Cropping: This involves taking a random crop of the image and then resizing it back to the original size. This can be seen as a combination of zooming and spatial location changes.
	\item Random Flipping: Randomly flip the image horizontally.
	\item Color Distortion: Apply a random color distortion. In the SimCLR paper, they use a combination of random brightness, random contrast, random saturation, and random hue changes. The strength of these distortions is controlled by a factor.
	\item Gaussian Blur: Apply a random Gaussian blur to the image. The extent of blurring is controlled by a factor.
\end{itemize}

\paragraph{MoCo v2 augmentation.}
For MoCo v2, the data augmentations are similar to those used in SimCLR, but there might be slight differences in implementation details. Here are the main augmentations used in MoCo v2:
\begin{itemize}
	\item Random Cropping: This involves taking a random crop of the image and then resizing it back to the original size.
	      Random Flipping: Randomly flip the image horizontally.
	\item Color Jitter: Randomly change the brightness, contrast, saturation, and hue of the image.
	\item Gaussian Blur: Apply Gaussian blur to the image with a certain probability.
	\item  Solarization: This is an augmentation introduced in MoCo v2. It inverts pixel values above a threshold, which can create a unique visual effect.
\end{itemize}

\paragraph{MAE augmentation.}
The core idea behind MAE is to mask out parts of an image and then train an autoencoder to reconstruct the original image from the masked version. This is somewhat analogous to the masked language modeling task used in models like BERT for NLP, where parts of the text are masked out and the model is trained to predict the masked words.

\section{Appendix: Details of Biology Experiments} \label{app_bio_exp}
\subsection{Dataset Setups}

Experiments are performed on biological datasets, including MC1374\footnote{https://bis.zju.edu.cn/MCA/}~\citep{han2018mapping}, GA1457\footnote{https://maayanlab.cloud/Harmonizome/gene/GAST}~\citep{rouillard2016harmonizome}, SAM\footnote{https://github.com/abbioinfo/CyAnno}~\citep{weber2016comparison}, , and HCL500\footnote{https://db.cngb.org/HCL/}~\citep{han2020construction} datasets.

To ensure a fair comparison, we first embed the data into a 2D space using the method under evaluation. We then assess the method's performance through 10-fold cross-validation. Classification accuracy is determined by applying a linear SVM classifier in the latent space, while clustering accuracy is gauged using k-means clustering in the same space. Further details about the datasets, baseline methods, and evaluation metrics can be found in Table \ref{at:5}.

\begin{table*}[h]
	\centering
	\caption{Datasets information of simple manifold embedding task}
	\begin{tabular}{c|cccc}
		\toprule
		{ Dataset } & { Train Samples } & { Test Samples } & Input Dimension & Class Number in label \\ \midrule
		{ MC1374 }  & { 24,000 }        & { 6,000 }        & { 1,374 }       & { 98 }                \\
		{ GA1457 }  & { 8,510 }         & { 2,127 }        & { 1,457 }       & { 49 }                \\
		{ SAM561 }  & { 69,491 }        & { 17,373 }       & { 561 }         & { 52 }                \\
		{ HCL500 }  & { 48,000 }        & { 12,000 }       & { 500 }         & { 45 }                \\
		\bottomrule
	\end{tabular}
	\label{at:5}
\end{table*}

\subsection{Baseline Methods and Implementation Details}

Dimension reduction methods that have been widely used on biological analyze are compared, including kPCA~\citep{halko2010finding}, Ivis~\citep{szubert_structure_preserving_2019}, PHATE~\citep{moon2019visualizing_}, PUMAP~\citep{sainburg_parametric_2021}, PaCMAP~\citep{Yingfan_pacmap}, DMTEV~\citep{zang2022evnet} and hNNE~\citep{sarfraz2022hierarchical}.

For DiffAug, both the semantic encoder \(\text{Enc}(\cdot)\), and the diffusion generator \(\text{Gen}(\cdot)\), are implemented using a Multi-Layer Perceptron (MLP). Their respective architectures are defined as: \(\text{Enc}(\cdot)\): [-1, 500, 300, 80].  The \(\text{Gen}(\cdot)\): is defined below,
\clearpage
{
\begin{lstlisting}[language=Python, caption=DiffusionModel for Biology Task]
    class AE(nn.Module):
    def __init__( self,in_dim, mid_dim=2000, time_step=1000,):
        super().__init__()
        self.enc1 = self.diff_block(in_dim, mid_dim)
        self.enc2 = self.diff_block(in_dim, mid_dim)
        self.enc3 = self.diff_block(in_dim, mid_dim)
        self.enc4 = self.diff_block(in_dim, mid_dim)

        self.dec1 = self.diff_block(in_dim, mid_dim)
        self.dec2 = self.diff_block(in_dim, mid_dim)
        self.dec3 = self.diff_block(in_dim, mid_dim)
        self.dec4 = self.diff_block(in_dim, mid_dim)
        self.time_encode = nn.Embedding(time_step, in_dim)
    
    def diff_block(in_dim, mid_dim):    
        return nn.Sequential(
        nn.LeakyReLU(), nn.InstanceNorm1d(in_dim),
        nn.Linear(in_dim, mid_dim), nn.LeakyReLU(),
        nn.InstanceNorm1d(mid_dim), nn.Linear(mid_dim, in_dim),)

    def forward(self, input, time, cond=None):
        input_shape = input.shape
        if len(input.size()) > 2:
            input = input.view(input.size(0), -1)
        ti = self.time_encode(time)
        cd = self.cond_model(cond).reshape(input.shape[0], -1)
        ee1 = self.enc1(input + ti + cd)
        ee2 = self.enc2(ee1 + ti+ cd) + ee1
        ee3 = self.enc3(ee2 + ti+ cd) + ee1 + ee2
        ee4 = self.enc4(ee3 + ti+ cd) + ee1 + ee2 + ee3

        ed1 = self.dec1(ee4 + ti+ cd)
        ed2 = self.dec2(ed1 + ti+ cd) + ee3 + ed1
        ed3 = self.dec3(ed2 + ti+ cd) + ee2 + ed1 + ed2
        ed4 = self.dec4(ed3 + ti+ cd) + ee1 + ed1 + ed2 + ed3
        return ed4.reshape(input_shape)
\end{lstlisting}
\tiny
}

To assess the efficacy of the proposed methods, following \cite{Yingfan_pacmap, sarfraz2022hierarchical}, we utilized linear SVM performance to evaluate the performance of differences methods. For the linear SVM evaluation, embeddings were partitioned with 90\% designated for training and 10\% for testing; the training set facilitated the linear SVM training, while the test set yielded the performance metrics. Detailed specifics of this configuration are elaborated in the Table \ref{app_param_time_bio}. 

\begin{table}[!htbp]
	\begin{center}
		\caption{Details of the training process in biological dataset.}
		\label{app_param_time_bio}
		\begin{tabular}{lcccccccc}
			\hline
			CF10  & $\nu$ & Learning Rate & Weight Decay & Batch Size & GPU    & Training Time & \\ \hline
			MC1374 & 1     & 0.0001        & 1e-6         & 300        & 1*V100 & 4.2 hours     & \\
			GA1457 & 1     & 0.0001        & 1e-6         & 300        & 1*V100 & 4.6 hours     & \\
			SAM561 & 1     & 0.0001        & 1e-6         & 300        & 1*V100 & 12.1 hours    & \\
			HCL500 & 0.1     & 0.0001        & 1e-6         & 300        & 1*V100 & 20.1 hours    & \\ \hline
		\end{tabular}
	\end{center}
\end{table}

\begin{figure}
	\centering
	\includegraphics[width=0.99\linewidth]{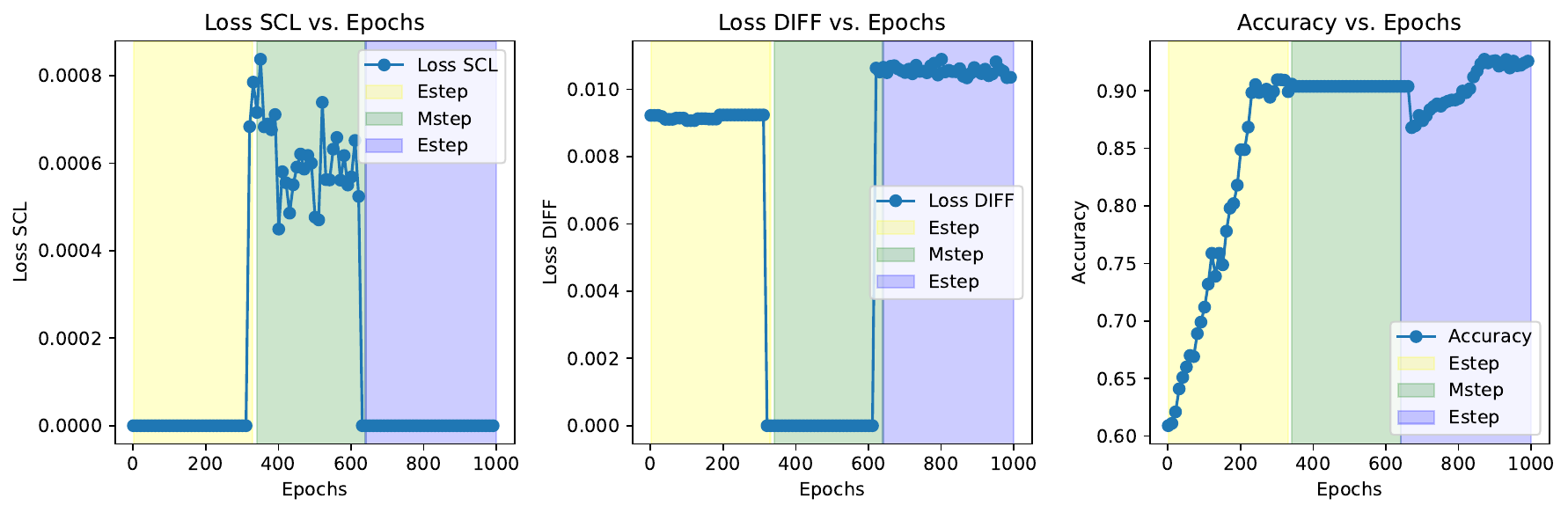}
	\caption{Training curves on the GA1457 dataset, including two Esteps and one Mstep. We can observe that the new generated data improves the correctness of E step.}
	\label{app_train_acc}
\end{figure}

\end{document}